\newif\ifaistatsformat
\aistatsformatfalse

\ifaistatsformat
\documentclass[twoside,letterpaper]{article}
\usepackage{aistats2027} 
\usepackage{amsmath,amssymb,amsthm}
\usepackage{bm}
\usepackage{booktabs}
\usepackage{adjustbox}
\usepackage{subcaption}
\usepackage{longtable}
\usepackage{nicefrac}
\usepackage{microtype}
\usepackage{graphicx}
\usepackage{xcolor}
\usepackage[numbers]{natbib}
\usepackage{algorithm}
\usepackage{algpseudocode}

\renewenvironment{algorithm}{\begin{algorithm*}}{\end{algorithm*}}

\else
\documentclass[10pt,letterpaper,reqno]{amsart}

\usepackage{amsmath,amssymb,amsthm}
\usepackage{bm}
\usepackage{booktabs}
\usepackage{adjustbox}
\usepackage{subcaption}
\usepackage{longtable}
\usepackage{nicefrac}
\usepackage{microtype}
\usepackage{graphicx}
\usepackage{xcolor}
\usepackage[numbers]{natbib}
\usepackage{algorithm}
\usepackage[asymmetric,left=1in,textwidth=16cm,top=1.62in,textheight=8in]{geometry}
\usepackage{algpseudocode}

\title[]{Training Dynamics of Softmax Self-Attention: \\ Fast Global Convergence via Preconditioning}
\author[]{Gautam Goel\textsuperscript{1}}
\author[]{Mahdi Soltanolkotabi\textsuperscript{2}}
\author[]{Peter Bartlett\textsuperscript{1}}
\thanks{\textsuperscript{1}Simons Institute, UC Berkeley}
\thanks{\textsuperscript{2}Department of Electrical and Computer Engineering, University of Southern California}

\fi

\algrenewcommand\algorithmicrequire{\textbf{Input:}}
\algrenewcommand\algorithmicensure{\textbf{Output:}}
\algnewcommand{\algorithmichyperparameters}{\textbf{Hyperparameters:}}
\algnewcommand{\Hyperparameters}{\item[\algorithmichyperparameters]}
\usepackage{mathtools,needspace,url}
\mathtoolsset{showonlyrefs}

\newtheorem{theorem}{Theorem}
\newtheorem{lemma}[theorem]{Lemma}

\DeclareMathOperator*{\Tr}{Tr}

\newcommand{\op}{\mathrm{op}}

\DeclareMathOperator{\polar}{polar}

\newcommand{\prepareheraappendix}{%
  \appendix
  \numberwithin{equation}{section}
  \numberwithin{theorem}{section}
  \allowdisplaybreaks[1]
  \setlength{\emergencystretch}{2em}
}

\ifaistatsformat\else
\usepackage{hyperref}
\fi

\begin{document}

\begin{abstract}
\noindent We study the training dynamics of gradient descent in a multilayer softmax self-attention network trained to perform linear regression and show that in the infinite-data limit this regression problem is equivalent to a matrix factorization problem. This connection motivates the design of Hera, a new first-order optimization algorithm for training self-attention networks. Key features of Hera include a reparameterization  of the model arising from the matrix factorization geometry and a rescaling of the parameters which ensures that the empirical loss gradients remain close to the auxiliary matrix factorization gradients throughout training.
Hera updates in the reparameterized model are equivalent to a certain preconditioned update in the original parameterization.
We also propose a variant of Hera called Hera-PSD;  our main theoretical result is that under small Gaussian initialization the excess risk of the model trained by Hera-PSD decreases exponentially fast in the number of gradient descent iterations, up to an additional error term that decreases as $1/n$, up to logarithmic factors, in the size of the training set.
A particularly appealing aspect of our excess risk bound is that it does not degrade as the number of self-attention layers increases. In a stylized time-series prediction experiment
in a nonlinear dynamical system, we find that a model trained by Hera achieves $98.1\%$ lower mean-squared test error than a model trained by AdamW and $97.7\%$ lower than a model trained by Muon at the same iteration budget. 

\end{abstract}

\maketitle

\section{Introduction}
The self-attention mechanism is a neural architecture originally proposed  for translation of natural language \cite{bahdanau2014neural} . It was subsequently adopted by \cite{vaswani2017attention} to form the basis of
the Transformer architecture, which underlies many recent advances in natural language processing, mathematical reasoning, and computer vision. Despite its numerous empirical successes, our theoretical understanding of the self-attention mechanism remains poor. Analysis of the training dynamics of gradient descent in a softmax self-attention network is exceedingly challenging due to the pairwise nonlinear interactions appearing in the softmax function and the coupled interactions between the gradient descent updates across different layers. For this reason, prior theoretical works suffer from three main limitations. 

First, many works study the training dynamics in settings where the softmax function is replaced by a linear approximation e.g., \citep{zhang2024trained, ahn2024transformers, ahn2023linear, zhang2024block, zhang2025dynamics}. Second, many works study an asymptotic setting, where the number of samples tends to infinity and the models' performance is measured in population loss, e.g., \citep{zhang2024trained, chen2024training, boursier2026softmax, huang2024incontext, zhang2024block, yang2024representations}. Third, of the works which retain the softmax nonlinearity, none study the coupled dynamics across multiple self-attention layers which are trained simultaneously. Many works instead focus on single-layer dynamics, e.g., \cite{chen2024training, duranthon2026statistical, yang2024representations, wu2023convergence, huang2024incontext, boursier2026softmax, arda2026finitetime, song2024unraveling, deora2024optimization, vasudeva2025implicit, zhao2024implicit, oymak2023role, mahdavi2024memorization}. Some recent works study multilayer dynamics where only one layer is trained at a time, e.g., \citep{nichani2024causal, chen2024induction}. In this work, we overcome these limitations and present the first finite-sample analysis of the coupled training dynamics in a self-attention network spanning multiple layers, without any linearization of the softmax function.

\subsection{Main Contributions}

We consider a linear regression setting where a multilayer softmax self-attention network is trained on covariate-response pairs generated from a planted linear model.
We propose a new family of optimization algorithms to train the finite-sample loss and characterize how the excess risk of the trained model scales as
a function of the number of training examples, the number of gradient descent iterations, and the number of model layers. 
Our results can be summarized as follows.

We show that the model parameters can be reparameterized in such a way that in the asymptotic limit as the number of samples tends to infinity, the training loss becomes a matrix factorization loss. In each layer, the reparameterization we introduce is defined recursively in terms of the covariances of the output tokens of the previous layer. Our reparameterization involves centering and whitening these output tokens, and can hence be viewed as performing a sort of normalization. Gradient descent in the reparameterized geometry can be interpreted as a certain preconditioned descent in the original geometry. 

We propose two simple gradient descent algorithms to the training loss. The first, which we call Hera, is simply gradient descent in the reparameterized coordinate system. While this algorithm is simple, we do not provide any theoretical guarantees for it. The second algorithm, called Hera-PSD, is a variant of Hera which additionally reparameterizes the model parameters to ensure a certain positive-definiteness condition in each layer, which is crucial for our theoretical analysis. Specifically, under this condition, we show that the gradient of the empirical loss concentrates around the gradient of the population loss, uniformly over a region containing its global minima.
This is the most technically challenging part of the paper due to the highly nonlinear nature of the softmax function. This concentration inequality controls the empirical training dynamics. In the course of proving this gradient result, we also establish novel concentration inequalities which bound how far the softmax-weighted mean and covariance can deviate from their Gaussian population targets; we believe these results may be of independent interest. 

Our main result (Theorem \ref{thm:multi-empirical}) is a finite-sample excess risk bound for Hera-PSD, which shows that the excess risk decays as $1/n$ in the number of samples (up to logarithmic factors) and $\exp(-ct)$ in the number of gradient descent steps $t$, for some explicit constant $c > 0$. To the best of our knowledge, this is the first global convergence result for a multilayer softmax self-attention network of arbitrary depth with simultaneous training across layers. Our excess risk bound does not deteriorate with the number of layers in the network, which is particularly appealing given recent interest in scaling depth in deep-learning models to hundreds or even thousands of layers, e.g., \citep{hayou2021stable, wang2024deepnet, yang2024depthmup}.

While our theoretical results all concern i.i.d. data generated from a fixed linear model, we believe that Hera can be useful in regression settings with nonstationary data. To validate this hypothesis, we benchmark Hera and Hera-PSD in a stylized time-series prediction experiment in a nonlinear dynamical system. Against AdamW and Muon, respectively, the test MSE reductions are $98.1\%$ and $97.7\%$ for Hera, and $87.7\%$ and $85.1\%$ for Hera-PSD. We intend this experiment to represent a proof-of-concept, and leave large scale experimentation for future work.

\section{Related Work}

\noindent\textbf{Linear self-attention training dynamics.}
The optimization landscape of linear self-attention in a linear regression was also studied by \cite{ahn2024transformers}, who characterized the optimal parameters for the linear self-attention mechanism but did not prove convergence of training dynamics to these optimal parameters. Prior empirical work of \cite{ahn2023linear} suggests that a theory of optimization behavior of linear attention may generalize to broader contexts. In-context learning of linear regression with a linearized attention mechanism was studied in \cite{zhang2024trained}, who proved a global convergence result on the population loss. We also note the work \cite{zhang2024block} which studied how a linear self-attention layer interacts with a linear MLP block in a regression setting. Zhang et al.~\citep{zhang2025dynamics} compare the training dynamics of linear attention with merged key--query weights against a parameterization that keeps the two factors separate, and show that these choices can produce markedly different learning trajectories. We also use a merged key-query matrix.

\noindent\textbf{Softmax self-attention training dynamics.}

A series of recent works, e.g., \cite{chen2024training, he2025demystified} study how a single-layer multihead softmax attention layer can learn to perform linear regression in-context; Yang et al.~\citep{yang2024representations} also study in-context learning in a  nonlinear regression setting.


Wu et al.~\citep{wu2023convergence} establish global convergence guarantees for gradient descent on encoder-only shallow transformers with softmax attention. They use overparameterization to obtain a Polyak--\L{}ojasiewicz inequality for the training loss, thus establishing fast convergence to global optimality. Song et al.~\citep{song2024unraveling} analyze empirical squared-loss optimization of one attention layer with softmax or Gaussian kernels. They establish geometric convergence for joint query, key, and value training under various initialization and overparameterization conditions.  The results of \cite{wu2023convergence} and \cite{song2024unraveling} are notable for establishing convergence on the finite-sample training loss.

Duranthon et al.~\citep{duranthon2026statistical} analyze single-layer softmax self-attention predictors for single-location regression in a high-dimensional limit. They show that such predictors can attain the Bayes risk at the population level, and characterize finite-sample test error using the replica method; this derivation is not rigorous, however. One difference between their setting and ours is that they consider a model where the dimension $d$ of the covariates is growing with the number of samples; in our setting, $d$ is fixed. Another series of recent works, e.g., \cite{vasudeva2025implicit, deora2024optimization, mahdavi2024memorization, tarzanagh2023transformers}, study training dynamics in classification settings.

All of the aforementioned works on softmax self-attention training dynamics are for single-layer models; the only works we are aware of which study multilayer settings are \citep{nichani2024causal} and \citep{ chen2024induction}, both of which study learning in discrete casual graphs. In both of these works, the self-attention layers are trained one at a time; these works do not study training dynamics coupled across layers. 

\noindent
\textbf{Softmax self-attention as a linear function.}
As observed in \cite{castin2025unified}, the softmax mean associated to a query  token converges to a linear function of the query when the number of tokens in the context window grows large, provided that the tokens are drawn i.i.d. from a fixed Gaussian distribution. This observation allows one to model the action of each successive model layer as applying a fixed linear transformation to each token, where the linear map depends on the parameters at that layer and the covariance of the input tokens at that layer. We exploit this fact to derive a reparameterization of parameters in which the training loss asymptotically approaches a matrix factorization loss. We note that \cite{castin2025unified} does not consider training dynamics, as we do here. 

Similarly to our work, Boursier and Boyer  \cite{boursier2026softmax} use this linearizing property to characterize the population loss in a linear regression setting, and prove a concentration bound which bounds how far the gradient of the training loss can deviate from the population gradient, obtaining finite-sample bounds for the training procedure. Their results are for an in-context learning setting, while we focus on the setting where all tokens are drawn from a fixed distribution. One key difference is that we obtain an exponential tail bound for the empirical gradient by controlling the norm of the key-query matrix throughout training; their concentration inequality is a much weaker $L_2$ result, effectively bounding only the variance of the gradient deviation.
We also study multilayer training dynamics, while they focus on the single-layer setting. Arda et al \cite{arda2026finitetime} also achieve finite-sample guarantees in a single softmax self-attention layer by bounding the discrepancy between the softmax function and a linear proxy. In contrast to our results, the descent procedure they analyze only converges at a slow $t^{-1/2}$ rate in the number of iterations; our bound has geometric optimization decay to a statistical floor. Our geometric guarantee assumes a symmetric positive definite residual target and uses a regularized objective in the lifted attention coordinates.

\noindent\textbf{Other related work.} We discuss matrix factorization, structure-aware and norm-constrained optimization, and depth scaling and trainability in Appendix~\ref{sec:additional-related-work}.

\section{Model}\label{sec:model}

\subsection{Data model}
We study regression from $d$-dimensional covariates to $p$-dimensional responses. The training data are $n$ independent pairs $\{(x_i,y_i)\}_{i=1}^n$ generated by
\[
x_i\sim\mathcal N(0,\Sigma),\qquad y_i=\Phi x_i+z_i,
\]
where $\Phi\in\mathbb R^{p\times d}$ is a fixed linear map, $\Sigma\in\mathbb R^{d\times d}$ is positive definite, and the noise increments $z_i\sim\mathcal N(0,\Omega)$ are independent across samples and independent of all covariates, with $\Omega\in\mathbb R^{p\times p}$ positive semidefinite.

\subsection{Residual self-attention}
We study self-attention networks with $L\ge1$ single-head self-attention layers. Each layer accepts tokens of dimension $d$ and also outputs tokens of dimension $d$; in the first layer, the input tokens during training are simply the samples $\{x_i\}_{i = 1}^n$. The vector $\xi_i^{(\ell)}\in\mathbb R^d$ is the representation of token $i$ after layer $\ell$. Each layer has a value matrix $U^{(\ell)}\in\mathbb R^{d\times d}$ and a key-query matrix $V^{(\ell)}\in\mathbb R^{d\times d}$; this matrix is also sometimes called a score matrix. Vaswani et al.~\citep{vaswani2017attention} learn separate query and key projections. Here we parameterize the bilinear attention score directly by $V^{(\ell)}$, absorbing the query--key product and score scaling into this matrix. Similar parameterizations appear in many prior works, e.g., \cite{castin2025unified, boursier2026softmax, zhang2024trained, wu2024pretraining}.

For $\ell=1,\ldots,L$, let $\bar\xi^{(\ell-1)}=n^{-1}\sum_{i=1}^n\xi_i^{(\ell-1)}$ be the mean input token. Define the attention scores $s_{ij}^{(\ell)}$ and weights $\pi_{ij}^{(\ell)}$ by
\begin{equation*}\label{eq:model-attention-weights}
\begin{aligned}
s_{ij}^{(\ell)}={}&(\xi_i^{(\ell-1)}-\bar\xi^{(\ell-1)})^\top\! V^{(\ell)}\cdot(\xi_j^{(\ell-1)}-\bar\xi^{(\ell-1)}),\\
\pi_{ij}^{(\ell)}={}&\frac{\exp(s_{ij}^{(\ell)})}{\sum_{a=1}^n\exp(s_{ia}^{(\ell)})}.
\end{aligned}
\end{equation*}
The softmax mean $\mu_i^{(\ell)}\in\mathbb R^d$ and the residual update are
\begin{equation}\label{eq:model-residual}
\begin{aligned}
\mu_i^{(\ell)}&=\sum_{j=1}^n\pi_{ij}^{(\ell)}(\xi_j^{(\ell-1)}-\bar\xi^{(\ell-1)}),\\
\xi_i^{(\ell)}&=\xi_i^{(\ell-1)}+U^{(\ell)}\mu_i^{(\ell)}.
\end{aligned}
\end{equation}
Thus each layer preserves its input through an identity skip connection and adds a learned attention correction. We note that all tokens used to compute the softmax weights  and the softmax means are centered. All attention sums include the self-key $j=i$. We denote the centered empirical covariance of the tokens entering layer $\ell$ by $C^{(\ell-1)}$. Centering is recomputed from the context entering each layer.

A linear regression head $W\in\mathbb R^{p\times d}$ maps the final token to a response prediction:
\begin{equation}\label{eq:model-head}
\hat y_i=W\xi_i^{(L)}.
\end{equation}
The readout acts on the unnormalized final token. In all theoretical results we set $p=d$ and fix $W=I_d$. Thus $\hat y_i=\xi_i^{(L)}$, and the raw parameters are $\theta=\{U^{(\ell)},V^{(\ell)}\}_{\ell=1}^L$. The forecasting experiment in Section~\ref{sec:experiments} uses a fixed coordinate selector with $p=3$.

The empirical training loss in the theoretical setting is
\begin{equation}\label{empirical-loss-definition}
\widehat{\mathcal L}(\theta)=\frac1{2n}\sum_{i=1}^n\|\xi_i^{(L)}-y_i\|_2^2,
\end{equation}
where each output is computed from the full training context by \eqref{eq:model-residual}. The skip connection already predicts $x_i$. We therefore define the residual target $X_\star=(\Phi-I_d)\Sigma^{1/2}$ and the irreducible risk $\mathcal L^\star=\tfrac12\Tr\Omega$. Our convergence theorems for Hera-PSD assume $X_\star$ is symmetric positive definite.

\section{Reparameterization}\label{sec:repar}

\subsection{Population loss}
We will first derive the single-layer population loss using the softmax weights without centering. This single-layer calculation essentially repeats an observation of \cite{castin2025unified}. 
Fix a query $x$ and let $x'_1,\ldots,x'_n$ be independent keys drawn from $\mathcal N(0,\Sigma)$. 
\[
\begin{aligned}
\mu(x)&=\sum_{i=1}^n\pi_i(x)x'_i,\\
\pi_i(x)&=\frac{\exp(x^\top Vx'_i)}{\sum_{j=1}^n\exp(x^\top Vx'_j)}.
\end{aligned}
\]
Taking expectations over the keys while holding the query fixed, we use the ratio-of-expectations approximation
\begin{equation}\label{eq:repar-ratio}
\begin{aligned}
\mathbb E[\mu(x)]&=\sum_{i=1}^n\mathbb E[\pi_i(x)x'_i]\\
&\approx\sum_{i=1}^n\frac{\mathbb E[\exp(x^\top Vx'_i)x'_i]}{\sum_{j=1}^n\mathbb E[\exp(x^\top Vx'_j)]}\\
&=\frac{\mathbb E_{x'}[x'\exp(x^\top Vx')]}{\mathbb E_{x'}[\exp(x^\top Vx')]}.
\end{aligned}
\end{equation}
The approximation in \eqref{eq:repar-ratio} motivates our definition of the population loss; we are able to show that the empirical attention mean converges to this ratio for each fixed query as $n\to\infty$. We denote the resulting infinite-context attention mean by $\mu_{\mathrm{pop}}(x)$.

To compute the two expectations, put $a=V^\top x$ and let $\varphi_\Sigma$ be the density of $\mathcal N(0,\Sigma)$. Completing the square in the Gaussian exponent gives
\[
\exp(a^\top x')\varphi_\Sigma(x')
=\exp(a^\top\Sigma a/2)\varphi_\Sigma(x'-\Sigma a).
\]
Integrating this identity, first alone and then after multiplication by $x'$, yields
\[
\begin{aligned}
\mathbb E_{x'}[\exp(a^\top x')]&=\exp(a^\top\Sigma a/2),\\
\mathbb E_{x'}[x'\exp(a^\top x')]&=\Sigma a\,\exp(a^\top\Sigma a/2).
\end{aligned}
\]
The exponential factors cancel, giving
\begin{equation}\label{eq:repar-population-mean}
\begin{aligned}
\mu_{\mathrm{pop}}(x)&=\Sigma V^\top x,\\
\xi^{(1)}(x)&=(I_d+U\Sigma V^\top)x.
\end{aligned}
\end{equation}
Thus the infinite-context attention mean is linear in the query, although finite-context softmax attention is nonlinear. With the identity readout, subtracting the skip connection from the target gives
\begin{align}
\mathcal L(\theta)&=\tfrac12\mathbb E\|x+U\mu_{\mathrm{pop}}(x)-y\|_2^2\notag\\
&=\mathcal L^\star+\tfrac12\|U\Sigma V^\top\Sigma^{1/2}-X_\star\|_F^2.\label{eq:repar-single-loss}
\end{align}
Thus the two attention matrices determine the entire fitted correction.

\subsection{Computing the multilayer population loss}
A linear image of a centered Gaussian vector is again centered Gaussian. Consequently, the same calculation applies inductively at every population layer, using the covariance of the preceding layer. Let $g=\Sigma^{-1/2}x\sim\mathcal N(0,I_d)$ and define the cumulative linear map $F^{(\ell)}$ by $\xi^{(\ell)}(x)=F^{(\ell)}g$. Starting from $F^{(0)}=\Sigma^{1/2}$ and $\Sigma^{(0)}=\Sigma$, define the population layer map $\Psi^{(\ell)}$ and update
\begin{equation}\label{eq:repar-covariance-recursion}
\begin{aligned}
\Psi^{(\ell)}&=I_d+U^{(\ell)}\Sigma^{(\ell-1)}V^{(\ell)\top},\\
F^{(\ell)}&=\Psi^{(\ell)}F^{(\ell-1)},\\
\Sigma^{(\ell)}&=\Psi^{(\ell)}\Sigma^{(\ell-1)}\Psi^{(\ell)\top}.
\end{aligned}
\end{equation}
The original parameters $\theta=\{U^{(\ell)},V^{(\ell)}\}_{\ell=1}^L$ therefore determine the loss through
\[
\mathcal L(\theta)=\mathcal L^\star+\frac12\|F^{(L)}-\Phi\Sigma^{1/2}\|_F^2.
\]

\subsection{Lifted coordinates}
A key contribution is a coordinate system with a simple matrix factorization reference objective. We refer to parameters in this new coordinate system as ``lifted" parameters.
We define the lifted parameters $(\tilde{U}, \tilde{V})$ from the original parameters $(U, V)$ using the covariance at each layer. Let $C^{(\ell-1)}\succ0$ be the centered empirical covariance of the input tokens to layer $\ell$, and let $P^{(\ell-1)}$ be their orthogonal alignment from \eqref{eq:v8-normalization}, with $P^{(0)}=I_d$. Write $T^{(\ell-1)}=(C^{(\ell-1)})^{1/2}P^{(\ell-1)}$. For a nonzero score matrix, define its normalized score scale $\gamma^{(\ell)}=\|T^{(\ell-1)\top}V^{(\ell)}T^{(\ell-1)}\|_F$ and choose a positive factor scale $r^{(\ell)}$. The lifted variables are
\begin{equation}\label{eq:repar-lifted-variables}
\begin{aligned}
\tilde U^{(\ell)}&=\frac{\gamma^{(\ell)}}{r^{(\ell)}}U^{(\ell)}T^{(\ell-1)},\\
\tilde V^{(\ell)}&=\frac{r^{(\ell)}}{\gamma^{(\ell)}}T^{(\ell-1)\top}V^{(\ell)}T^{(\ell-1)}.
\end{aligned}
\end{equation}
Then $\|\tilde V^{(\ell)}\|_F=r^{(\ell)}$. The factor scale specifies the reciprocal scaling of the two lifted matrices. Hera holds $\gamma^{(\ell)}$ fixed, updates the lifted factors, and obtains the raw weights by inverting \eqref{eq:repar-lifted-variables}.
Specifically, the weights $(U^{(\ell)},V^{(\ell)})$ can be recovered from $(\tilde U^{(\ell)},\tilde V^{(\ell)})$ by the formulas
\[
\begin{aligned}
U^{(\ell)}&=\frac{r^{(\ell)}}{\gamma^{(\ell)}}\tilde U^{(\ell)}\bigl(T^{(\ell-1)}\bigr)^{-1},\\
V^{(\ell)}&=\frac{\gamma^{(\ell)}}{r^{(\ell)}}\bigl(T^{(\ell-1)}\bigr)^{-\top}\tilde V^{(\ell)}\bigl(T^{(\ell-1)}\bigr)^{-1}.
\end{aligned}
\]
Hera updates the two lifted matrices independently. Hera-PSD ties them in every attention layer. It sets $\tilde V^{(\ell)}=B^{(\ell)}$ and reconstructs $\tilde U^{(\ell)}$ from $B^{(\ell)},K^{(\ell)}$ by the bounded PSD parameterization in \eqref{eq:multi-parameterization}. The attention factors $\{B^{(\ell)},K^{(\ell)}\}_{\ell=1}^L$ then determine every prediction. 

\subsection{Matrix Factorization Interpretation}

To define the auxiliary loss $\widetilde{L}$, reconstruct raw weights using population covariances and alignments. In one layer this uses $T^{(0)}=\Sigma^{1/2}$, so the corresponding raw parameters are
\begin{equation}\label{eq:repar-population-weights}
\begin{aligned}
U&=\frac r\gamma\tilde U\Sigma^{-1/2},\\
V&=\frac\gamma r\Sigma^{-1/2}\tilde V\Sigma^{-1/2}.
\end{aligned}
\end{equation}
Substituting \eqref{eq:repar-population-weights} gives $U\Sigma V^\top\Sigma^{1/2}=\tilde U\tilde V^\top$: the factors $r/\gamma$ and $\gamma/r$ cancel. Thus \eqref{eq:repar-single-loss} becomes
\[
\mathcal L(U,V)=\mathcal L^\star+\frac12\|\tilde U\tilde V^\top-X_\star\|_F^2.
\]
Under this population reconstruction, the score scale $\gamma$ changes the attention scores without changing the loss at fixed lifted parameters.

For multiple layers, set $S^{(k)}=\Sigma^{1/2}+\sum_{\ell=1}^k\tilde U^{(\ell)}\tilde V^{(\ell)\top}$, with $S^{(0)}=\Sigma^{1/2}$. The parameterization in \eqref{eq:multi-parameterization} makes every summand PSD, so $S^{(k)}\succeq\Sigma^{1/2}\succ0$ for all $k=0,\ldots,L$. We show inductively that $\xi^{(\ell)}(x)=S^{(\ell)}g$, where $g=\Sigma^{-1/2}x$, and that each aligned attention input equals $g$. If the input to layer $\ell$ is $S^{(\ell-1)}g$, its covariance is $\Sigma^{(\ell-1)}=S^{(\ell-1)}S^{(\ell-1)\top}$, and whitening gives $(\Sigma^{(\ell-1)})^{-1/2}S^{(\ell-1)}g$. In general the coefficient of $g$ is orthogonal and is removed by polar alignment. Specifically,
\[
\begin{aligned}
P^{(\ell-1)}&=(\Sigma^{(\ell-1)})^{-1/2}S^{(\ell-1)},\\
T^{(\ell-1)}&=(\Sigma^{(\ell-1)})^{1/2}P^{(\ell-1)}=S^{(\ell-1)}.
\end{aligned}
\]
These identities also hold at the input, where $P^{(0)}=I_d$. For our PSD increments, the reference maps are symmetric positive definite, so the population alignments are in fact $P^{(\ell-1)}=I_d$; the empirical alignments need not be.
Substituting $T^{(\ell-1)}=S^{(\ell-1)}$ into the layerwise reconstruction formulas, the residual branch satisfies
\[
U^{(\ell)}\Sigma^{(\ell-1)}V^{(\ell)\top}\xi^{(\ell-1)}(x)=\tilde U^{(\ell)}\tilde V^{(\ell)\top}g.
\]
Adding the skip connection gives $\xi^{(\ell)}(x)=S^{(\ell)}g$. Thus the population reconstruction has loss $\mathcal L(\theta)=\widetilde{L}(\tilde\theta)$, where we define
\begin{equation}\label{eq:repar-multilayer-lifted-loss}
\widetilde{L}(\tilde\theta)=\mathcal L^\star+\frac12\left\|\sum_{\ell=1}^L\tilde U^{(\ell)}\tilde V^{(\ell)\top}-X_\star\right\|_F^2.
\end{equation}
Here $\tilde\theta$ denotes the lifted attention factors. The sum fits the residual target because the initial $\Sigma^{1/2}g$ reaches the output through the skip connections. Hera and Hera-PSD instead reconstruct raw weights using empirical covariances, so $\widetilde{L}(\tilde\theta)$ need not equal the population loss $\mathcal L(\theta)$ at those weights.
In our analysis of Hera-PSD, we keep these objectives distinct: the proofs first control the auxiliary loss and then compare it with the population risk. The notation $\widehat{\mathcal L}(\tilde\theta)$ denotes the empirical loss composed with the empirical reconstruction.

\section{Hera and Hera-PSD}\label{sec:frozen-gaussian}\label{sec:multi-layer}
The prediction in \eqref{eq:repar-multilayer-lifted-loss} depends on the sum of the attention increments. In Hera (Algorithm~\ref{alg:hera-unrestricted}), the lifted matrices are free parameters in every layer. Our convergence theorem concerns Hera-PSD (Algorithm~\ref{alg:hera-psd}). In every layer $\ell=1,\ldots,L$, Hera-PSD parameterizes that increment with two $d\times d$ matrices $B^{(\ell)},K^{(\ell)}$:
\begin{equation}\label{eq:multi-parameterization}
\begin{aligned}
 S(K)&=(I_d+K^\top K)^{-1},\qquad \tilde V^{(\ell)}=B^{(\ell)},\\
 \tilde U^{(\ell)}&=B^{(\ell)}S(K^{(\ell)}).
\end{aligned}
\end{equation}
Since $0\prec S(K)\preceq I_d$, the increment
\begin{equation}\label{eq:multi-psd-contribution}
A^{(\ell)}=\tilde U^{(\ell)}\tilde V^{(\ell)\top}=B^{(\ell)}S(K^{(\ell)})B^{(\ell)\top}
\end{equation}
is positive semidefinite. Consequently, each reference map $S^{(k)}=\Sigma^{1/2}+\sum_{\ell\le k}A^{(\ell)}$ has smallest singular value at least $\sigma_{\min}(\Sigma^{1/2})$. The attainable residuals are positive semidefinite. We assume $X_\star\succ0$ so that the target is representable by the class of predictors Hera-PSD trains.

Hera-PSD takes simultaneous gradient steps directly in $B^{(\ell)},K^{(\ell)}$ on the empirical prediction loss plus a regularizer $\mathcal R$. The two factors have constant learning rates $\eta_B$ and $\eta_K$, respectively, shared across layers. Each gradient differentiates through the full empirical forward pass, including the covariance and polar alignment.

For the theorem we use shifted Gram penalties, as in regularized matrix factorization \citep{tu16,tarmoun2021understanding}. In the earlier layers, balance against the bounded gate controls the sum of their PSD increments. The last layer uses the same $S(K)$ gate and a different Gram penalty so it can fit the remaining positive definite residual. Every balanced choice of the earlier gates and orientations extends to an exact fit with zero penalty (Appendix~\ref{sec:multi-psd}).

The following theorem applies at every depth $L\ge1$. With $W=I_d$, the attention layers fit the residual target through the sum of their PSD increments. For $L=1$, this sum reduces to $B^{(1)}S(K^{(1)})B^{(1)\top}$.

\begin{theorem}[Finite-sample convergence in $L$ layers]\label{thm:frozen-gaussian}\label{thm:multi-empirical}
Let $L\ge1$ and fix $p=d$, $W=I_d$. Assume $\Sigma\succ0$ and that $X_\star=(\Phi-I_d)\Sigma^{1/2}$ is symmetric positive definite. Fix $\delta\in(0,1)$ and use the balanced Gaussian initialization, Gram regularizer~\eqref{eq:fg-reg}, constant gradient steps in $B,K$, and sample conditions of Appendix~\ref{app:frozen-gaussian}. With probability at least $1-\delta$, the raw parameters $\theta_t=\{U_t^{(\ell)},V_t^{(\ell)}\}_{\ell=1}^L$ satisfy, simultaneously for every $t\ge0$,
\begin{equation}\label{eq:multi-risk}
 \mathcal L(\theta_t)-\mathcal L^\star
 \le \|X_\star\|_F^2\exp(-ct)+K_{\rm risk}\,d^2q/n .
\end{equation}
Here $q=2d^2\log(2n)+\log(16/\delta)$ and $\mathcal L^\star=\tfrac12\Tr\Omega$. The constants $c$ and $K_{\rm risk}$ are given in \eqref{eq:direct-rates} and \eqref{eq:multi-risk-constant}. The rate, statistical constant, and sufficient sample threshold are independent of $L$. The constants depend polynomially on the population scales and their indicated inverses, up to logarithms in the sample threshold.
\end{theorem}
\begin{proof}[Proof overview]
\ifdefined\standaloneGaussianProof The original-sample softmax mean and covariance bounds\else Lemmas~\ref{lem:softmax-mean-concentration} and~\ref{lem:softmax-covariance-concentration}\fi{} control the mean and its derivative at the initial whitened tokens. At layer $\ell$, compare the tokens to a linear image of those initial tokens. The reference map starts at the square root of the empirical input covariance and adds the PSD increments of the preceding layers. Its least singular value stays bounded below. Whitening and polar alignment undo this reference map exactly, so only the nonlinear remainder perturbs the next whitened tokens. Induction bounds that remainder and its derivative. The cumulative size of the increments before the final layer controls propagation; the final layer needs no further whitening.

The earlier Gram penalties keep the total contribution before the last layer small. The last layer therefore sees a positive definite residual target at every step. Its two factor gradients control prediction error and final-layer imbalance, while the earlier value-factor gradients control the remaining Gram residuals. This gives a Polyak--{\L}ojasiewicz inequality for the auxiliary objective. The empirical steps preserve the singular-value bounds used in both estimates, so the same concentration event suffices for all iterations. A final comparison using the true covariance recursion gives the stated risk at the empirically reconstructed weights.
\end{proof}

\section{Experiments}\label{sec:experiments}
\begin{figure*}[t]
\centering
\includegraphics[width=\textwidth]{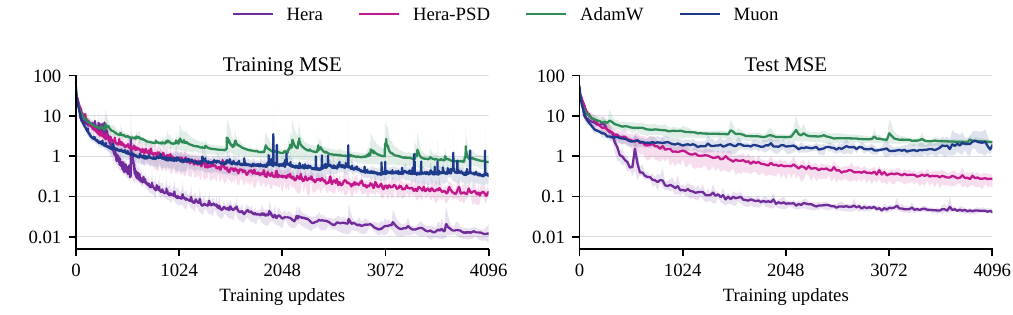}
\caption{Lorenz--63 one-step forecasting. Lines show mean training MSE (left) and test MSE (right) over five random initializations for each validation-selected configuration.} 
\label{fig:lorenz-training-test}
\end{figure*}


We compare two lifted parameterizations on one-step forecasting of the Lorenz--63 system \citep{lorenz1963deterministic}. Each three-layer model takes five consecutive observed states as input and uses the same fixed readout to predict the next three-dimensional state. In this experiment, \emph{Hera} updates the lifted matrices $\tilde U$ and $\tilde V$ independently, while \emph{Hera-PSD} uses $\tilde V=B$ and $\tilde U=B(I+K^\top K)^{-1}$ in all three layers, and updates $B$ and $K$ in each layer. Both algorithms are evaluated without any regularization. AdamW and Muon train the raw attention matrices.

We use successive blocks of 500 training, 500 validation, and 500 test examples. For each method, we screen 44 configurations for 500 updates with five random seeds, each corresponding to a fresh initialization. We then restart the ten best eligible configurations for 4,096 updates with five fresh seeds. We select a winning configuration by mean endpoint validation MSE, requiring all five seeds to succeed in each stage. At 4,096 updates, mean test MSE is $0.04135$ for Hera, $0.27297$ for Hera-PSD, $2.22217$ for AdamW, and $1.83235$ for Muon (Figure~\ref{fig:lorenz-training-test}). On this trajectory, freeing the lifted factors gives the lowest error, while Hera-PSD also improves on both raw-coordinate optimizers.  Appendix~\ref{sec:experiments-appendix} gives the search grid, selected settings, and forecasts.

\section{Conclusion}
In this work, we identified three key algorithmic design principles to guide the design of optimization algorithms for regression models based on self-attention. 

First, we propose a new parameterization of multilayer self-attention networks which arises from the observation that in the asymptotic limit, the regression problem solved by the model is in fact a deep matrix factorization problem. We propose a new parametrization where each of the original variables is defined in terms of a ``lifted" variable, which in turn is defined in terms of the covariance of the tokens received by that layer. Each of these covariances can be recursively computed from the parameters and token covariances associated to previous layers.
Optimization algorithms to train the self-attention parameters can hence be derived from  matrix factorization algorithms.

Second, we suggest that the key-query matrices  should be constrained to have fixed norm over training in the lifted coordinate system, and this norm should decay logarithmically in the number of training samples, so as to ensure concentration of the empirical gradient around the population gradient. We further propose that in the lifted coordinate system, the lifted value matrix in each layer should be rescaled inversely with the scaling of the lifted key-query matrix in that layer; these rescalings cancel in the asymptotic limit as the number of samples tends to infinity, so that there is no norm constraint on the lifted parameters.

Third, we suggest that the lifted parameters in each hidden layer should be chosen to guarantee that their product is positive semidefinite; this can be achieved by a further reparameterization of the lifted parameters, as in Hera-PSD. This condition guarantees that the reparameterization is well-defined throughout training and ensures concentration of the empirical gradient around the population gradient.

We identify several promising directions for future work. First, it would be useful to extend our excess risk bound for i.i.d. data to settings where the data is instead generated from a weakly mixing dynamical system; this setting is natural for time-series forecasting. Second, it is natural to ask what performance gains could be achieved using adaptive stepsizes or momentum. Lastly, we believe that the reparameterization we derive may be useful to understand self-attention training dynamics in other learning settings, such as in-context learning or reinforcement learning.

\ifaistatsformat
\subsection*{AI use statement}
In this work, we used generative AI tools for running experiments, generating synthetic data for our simulations in the Lorenz--63 system, proving and checking theorems, and assisting in writing proofs.
We have not used generative AI tools for other tasks with required disclosure, and the rest of the required disclosure tasks are not applicable to this work.
Additionally, we used generative AI tools for editing and formatting our paper, and for identifying prior literature which is relevant to our work.
We have reviewed all AI-assisted work. We have carefully evaluated the accuracy of all proofs.
We take responsibility for the final content of this work, including text, claims, or artifacts produced with the aid of generative AI.

\fi
\bibliographystyle{alpha}
\bibliography{references}
\ifaistatsformat
\clearpage
\section*{Checklist}

\begin{enumerate}

  \item For all models and algorithms presented, check if you include:
  \begin{enumerate}
    \item A clear description of the mathematical setting, assumptions, algorithm, and/or model. [Yes]
    \item An analysis of the properties and complexity (time, space, sample size) of any algorithm. [Yes]
    \item (Optional) Anonymized source code, with specification of all dependencies, including external libraries. [No] We will make this available upon acceptance.
  \end{enumerate}

  \item For any theoretical claim, check if you include:
  \begin{enumerate}
    \item Statements of the full set of assumptions of all theoretical results. [Yes]
    \item Complete proofs of all theoretical results. [Yes]
    \item Clear explanations of any assumptions. [Yes]
  \end{enumerate}

  \item For all figures and tables that present empirical results, check if you include:
  \begin{enumerate}
    \item The code, data, and instructions needed to reproduce the main experimental results (either in the supplemental material or as a URL). [No] We will make this available upon acceptance.
    \item All the training details (e.g., data splits, hyperparameters, how they were chosen). [Yes]
    \item A clear definition of the specific measure or statistics and error bars (e.g., with respect to the random seed after running experiments multiple times). [Yes]
    \item A description of the computing infrastructure used. (e.g., type of GPUs, internal cluster, or cloud provider). [No]
  \end{enumerate}

  \item If you are using existing assets (e.g., code, data, models) or curating/releasing new assets, check if you include:
  \begin{enumerate}
    \item Citations of the creator If your work uses existing assets. [Not Applicable]
    \item The license information of the assets, if applicable. [Not Applicable]
    \item New assets either in the supplemental material or as a URL, if applicable. [Not Applicable]
    \item Information about consent from data providers/curators. [Not Applicable]
    \item Discussion of sensible content if applicable, e.g., personally identifiable information or offensive content. [Not Applicable]
  \end{enumerate}

  \item If you used crowdsourcing or conducted research with human subjects, check if you include:
  \begin{enumerate}
    \item The full text of instructions given to participants and screenshots. [Not Applicable]
    \item Descriptions of potential participant risks, with links to Institutional Review Board (IRB) approvals if applicable. [Not Applicable]
    \item The estimated hourly wage paid to participants and the total amount spent on participant compensation. [Not Applicable]
  \end{enumerate}

\end{enumerate}
\fi
\prepareheraappendix
\onecolumn
\noindent Appendix~\ref{sec:algorithm} gives pseudocode for Hera and Hera-PSD. Appendix~\ref{sec:additional-related-work} discusses related work. Appendix~\ref{sec:multi} proves the Hera-PSD convergence theorem for every $L\ge1$. Appendix~\ref{sec:supporting-lemmas} proves the softmax and normalization estimates.

\section{Hera and Hera-PSD algorithms}\label{sec:algorithm}
Hera (Algorithm~\ref{alg:hera-unrestricted}) updates the independent lifted matrices $\tilde\theta_t=\{\tilde U_t^{(\ell)},\tilde V_t^{(\ell)}\}_{\ell=1}^L$ using the empirical prediction loss with fixed readout $W$. Hera-PSD (Algorithm~\ref{alg:hera-psd}) uses the tie in \eqref{eq:multi-parameterization}. The theorem uses $W=I_d$ and direct gradient steps on $B^{(\ell)},K^{(\ell)}$ with the Gram penalties in \eqref{eq:fg-reg}. Both algorithms differentiate the full forward pass at the current iterate before updating any layer.

\begingroup
\makeatletter
\let\hera@ruled\fs@ruled
\def\fs@ruled{\hera@ruled\let\@fs@iftopcapt\iffalse}
\makeatother
\floatstyle{ruled}\restylefloat{algorithm}
\newcommand{\HeraInputWhitening}{%
\State $C^{(0)}\gets\Call{Covariance}{\{x_i\}_{i=1}^n}$
\State $\{\tilde\xi_i^{(0)}\}_{i=1}^n\gets\Call{Normalize}{\{x_i\}_{i=1}^n}$
}
\newcommand{\HeraForwardPass}{%
\State $C_t^{(0)}\gets C^{(0)}$
\State $P_t^{(0)}\gets I_d$
\State $\xi_{i,t}^{(0)}\gets x_i$ for all $i$
\State $\tilde\xi_{i,t}^{(0)}\gets\tilde\xi_i^{(0)}$ for all $i$
\For{$\ell=1,\ldots,L$} \Comment{Forward pass}
    \State $r_t^{(\ell)}\gets\|\tilde V_t^{(\ell)}\|_F$
    \State $U_t^{(\ell)}\gets(r_t^{(\ell)}/\gamma)\tilde U_t^{(\ell)}P_t^{(\ell-1)\top}\bigl(C_t^{(\ell-1)}\bigr)^{-1/2}$
    \State $V_t^{(\ell)}\gets(\gamma/r_t^{(\ell)})\bigl(C_t^{(\ell-1)}\bigr)^{-1/2}P_t^{(\ell-1)}\tilde V_t^{(\ell)}P_t^{(\ell-1)\top}\bigl(C_t^{(\ell-1)}\bigr)^{-1/2}$
    \State $\{\xi_{i,t}^{(\ell)}\}_{i=1}^n\gets\Call{Softmax-Self-Attention}{\{\xi_{i,t}^{(\ell-1)}\}_{i=1}^n,U_t^{(\ell)},V_t^{(\ell)}}$ \Comment{Compute output tokens}
    \If{$\ell<L$}
        \State $C_t^{(\ell)}\gets\Call{Covariance}{\{\xi_{i,t}^{(\ell)}\}_{i=1}^n}$
        \State $\{w_{i,t}^{(\ell)}\}_{i=1}^n\gets\Call{Normalize}{\{\xi_{i,t}^{(\ell)}\}_{i=1}^n}$
        \State $P_t^{(\ell)}\gets\Call{Polar}{n^{-1}\sum_i w_{i,t}^{(\ell)}\tilde\xi_{i,t}^{(\ell-1)\top}}$
        \State $\tilde\xi_{i,t}^{(\ell)}\gets P_t^{(\ell)\top}w_{i,t}^{(\ell)}$ for all $i$
    \EndIf
\EndFor
}
\begin{algorithm}[H]
\captionsetup{labelfont=bf,textfont=bf}
\footnotesize
\begin{algorithmic}[1]
\Statex \textbf{Hyperparameters:} Updates $T$, step size $\eta>0$, score scale $\gamma>0$
\Statex \textbf{Input:} Data $\{(x_i,y_i)\}_{i=1}^n$, fixed readout $W$, initial lifted matrices $\{\tilde U_0^{(\ell)},\tilde V_0^{(\ell)}\}_{\ell=1}^L$, regularizer $\mathcal R$ on the lifted matrices
\Statex \textbf{Output:} Raw weights $\{U_T^{(\ell)},V_T^{(\ell)}\}_{\ell=1}^L$
\Statex
\HeraInputWhitening
\For{$t=0,\ldots,T$}
    \HeraForwardPass
    \If{$t<T$} \Comment{Simultaneous gradient step in all lifted matrices}
        \State $\tilde\theta_t\gets\{(\tilde U_t^{(\ell)},\tilde V_t^{(\ell)})\}_{\ell=1}^L$
        \State $\hat y_{i,t}\gets W\xi_{i,t}^{(L)}$ for all $i$
        \State $\widehat{\mathcal L}(\tilde\theta_t)\gets(2n)^{-1}\sum_i\|\hat y_{i,t}-y_i\|_2^2$
        \State $\widehat{\mathcal Q}(\tilde\theta_t)\gets\widehat{\mathcal L}(\tilde\theta_t)+\mathcal R(\tilde\theta_t)$
        \State $\tilde U_{t+1}^{(\ell)}\gets\tilde U_t^{(\ell)}-\eta\nabla_{\tilde U^{(\ell)}}\widehat{\mathcal Q}(\tilde\theta_t)$ for all $\ell$
        \State $\tilde V_{t+1}^{(\ell)}\gets\tilde V_t^{(\ell)}-\eta\nabla_{\tilde V^{(\ell)}}\widehat{\mathcal Q}(\tilde\theta_t)$ for all $\ell$
    \EndIf
\EndFor
\State \Return $\{U_T^{(\ell)},V_T^{(\ell)}\}_{\ell=1}^L$
\end{algorithmic}
\caption{Hera }\label{alg:hera-unrestricted}
\end{algorithm}

\begin{algorithm}[H]
\captionsetup{labelfont=bf,textfont=bf}
\footnotesize
\begin{algorithmic}[1]
\Statex \textbf{Hyperparameters:} Updates $T$, score scale $\gamma>0$, constant step sizes $\eta_B,\eta_K>0$
\Statex \textbf{Input:} Data $\{(x_i,y_i)\}_{i=1}^n$, initial factors $\{B_0^{(\ell)},K_0^{(\ell)}\}_{\ell=1}^L$, differentiable regularizer $\mathcal R$ on these factors
\Statex \textbf{Output:} Raw weights $\{U_T^{(\ell)},V_T^{(\ell)}\}_{\ell=1}^L$
\Statex
\HeraInputWhitening
\For{$t=0,\ldots,T$}
    \For{$\ell=1,\ldots,L$}
        \State $\tilde V_t^{(\ell)}\gets B_t^{(\ell)}$
        \State $\tilde U_t^{(\ell)}\gets B_t^{(\ell)}\big[I_d+(K_t^{(\ell)})^\top K_t^{(\ell)}\big]^{-1}$
    \EndFor
    \HeraForwardPass
    \If{$t<T$} \Comment{Simultaneous gradient step in all factors}
        \State $\vartheta_t\gets\{(B_t^{(\ell)},K_t^{(\ell)})\}_{\ell=1}^L$
        \State $\widehat{\mathcal L}(\vartheta_t)\gets(2n)^{-1}\sum_i\|\xi_{i,t}^{(L)}-y_i\|_2^2$
        \State $\widehat{\mathcal Q}(\vartheta_t)\gets\widehat{\mathcal L}(\vartheta_t)+\mathcal R(\vartheta_t)$
        \State $B_{t+1}^{(\ell)}\gets B_t^{(\ell)}-\eta_B\nabla_{B^{(\ell)}}\widehat{\mathcal Q}(\vartheta_t)$ for all $\ell$
        \State $K_{t+1}^{(\ell)}\gets K_t^{(\ell)}-\eta_K\nabla_{K^{(\ell)}}\widehat{\mathcal Q}(\vartheta_t)$ for all $\ell$
    \EndIf
\EndFor
\State \Return $\{U_T^{(\ell)},V_T^{(\ell)}\}_{\ell=1}^L$
\end{algorithmic}
\caption{Hera-PSD}\label{alg:hera-psd}
\end{algorithm}
\endgroup

\subsection{Forward-pass subroutines}\label{subsec:hera-subroutines}
Both algorithms use the same three preprocessing subroutines. $\operatorname{Polar}(A)$ (Algorithm~\ref{alg:polar}) returns the orthogonal polar factor $\polar(A)=A(A^\top A)^{-1/2}$ of an invertible square matrix. $\operatorname{Covariance}(\{\xi_i\}_{i=1}^n)$ (Algorithm~\ref{alg:covariance}) returns the centered empirical covariance with normalization $1/n$. $\operatorname{Normalize}(\{\xi_i\}_{i=1}^n)$ (Algorithm~\ref{alg:normalize}) returns the centered, whitened tokens $\{w_i\}_{i=1}^n$ before alignment. After each attention layer $\ell<L$, each algorithm stores these tokens as $w_{i,t}^{(\ell)}$, computes $P_t^{(\ell)}$ from their cross moment with the preceding aligned tokens $\tilde\xi_{i,t}^{(\ell-1)}$, and defines $\tilde\xi_{i,t}^{(\ell)}=P_t^{(\ell)\top}w_{i,t}^{(\ell)}$, exactly as in \eqref{eq:v8-normalization}. Each indexed quantity is assigned once, and subroutine variables are local to each call. At the input, $P^{(0)}=I_d$, so the output of $\operatorname{Normalize}$ is assigned directly to $\tilde\xi_i^{(0)}$.
Both methods require nonsingular empirical covariances and polar inputs for whitening and alignment, and a nonzero lifted score factor. Theorem \ref{thm:multi-empirical-detailed} verifies that these conditions hold for Hera-PSD; we do not provide any theoretical guarantees for Hera. 

The call $\operatorname{Softmax\text{-}Self\text{-}Attention}(\{\xi_i\}_{i=1}^n,U,V)$ returns all raw residual outputs in \eqref{eq:v8-forward-normalized}, using the supplied raw array and weights: it centers the queries, keys, and attention values, includes every self-key, and adds the uncentered skip $\xi_i$. Every subroutine and alignment multiplication are differentiated in full. The fixed readout acts directly on the raw output of layer $L$. The index $t$ denotes a training iteration and $\ell$ a layer, so the forward call maps layer $\ell-1$ to layer $\ell$ at fixed $t$.

\begingroup
\floatstyle{ruled}\restylefloat{algorithm}
\begin{algorithm}[H]
\caption{Polar}\label{alg:polar}
\footnotesize
\begin{algorithmic}[1]
\Require Invertible matrix $A\in\mathbb R^{d\times d}$
\Ensure Orthogonal polar factor $P=\polar(A)$
\Statex
\State $G\gets A^\top A$
\State $P\gets A G^{-1/2}$ \Comment{Principal inverse square root}
\State \Return $P$
\end{algorithmic}
\end{algorithm}

\begin{algorithm}[H]
\caption{Covariance}\label{alg:covariance}
\footnotesize
\begin{algorithmic}[1]
\Require Tokens $\{\xi_i\}_{i=1}^n\subset\mathbb R^d$
\Ensure Centered empirical covariance $C$
\Statex
\State $\bar\xi\gets n^{-1}\sum_i\xi_i$
\State $C\gets n^{-1}\sum_i(\xi_i-\bar\xi)(\xi_i-\bar\xi)^\top$
\State \Return $C$
\end{algorithmic}
\end{algorithm}

\begin{algorithm}[H]
\caption{Normalize}\label{alg:normalize}
\footnotesize
\begin{algorithmic}[1]
\Require Tokens $\{\xi_i\}_{i=1}^n\subset\mathbb R^d$ with positive definite empirical covariance
\Ensure Centered, whitened tokens $\{w_i\}_{i=1}^n$ before alignment
\Statex
\State $\bar\xi\gets n^{-1}\sum_i\xi_i$
\State $C\gets\Call{Covariance}{\{\xi_i\}_{i=1}^n}$
\State $w_i\gets C^{-1/2}(\xi_i-\bar\xi)$ for all $i$ \Comment{Principal inverse square root}
\State \Return $\{w_i\}_{i=1}^n$
\end{algorithmic}
\end{algorithm}

\endgroup

\subsection{Per-iteration complexity}\label{sec:hera-complexity}
Both Hera and Hera-PSD cost $O(L(n^2d+nd^2+d^3))$ arithmetic operations per iteration. The terms account for dense attention, token covariances, and layerwise matrix functions and their derivatives. Storing the attention weights, tokens, and factors for differentiation takes $O(L(n^2+nd+d^2))$ memory.

\subsection{Empirical forward maps}\label{sec:multi-notation}
Set
\begin{equation}\label{eq:data}
C^{(0)}=\frac1n\sum_i(x_i-\bar x)(x_i-\bar x)^\top,\qquad
\bar x=\frac1n\sum_i x_i.
\end{equation}
The initial aligned tokens are $\tilde\xi_i^{(0)}=(C^{(0)})^{-1/2}(x_i-\bar x)$, with $P^{(0)}=I_d$. For each intermediate layer, centering, whitening, and alignment mean
\begin{equation}\label{eq:v8-normalization}
\begin{aligned}
\bar\xi^{(k)}&=n^{-1}\sum_i\xi_i^{(k)},&C^{(k)}&=n^{-1}\sum_i(\xi_i^{(k)}-\bar\xi^{(k)})(\xi_i^{(k)}-\bar\xi^{(k)})^\top,\\
w_i^{(k)}&=(C^{(k)})^{-1/2}(\xi_i^{(k)}-\bar\xi^{(k)}),&P^{(k)}&=\polar\!\left(n^{-1}\sum_iw_i^{(k)}\tilde\xi_i^{(k-1)\top}\right),\\
\tilde\xi_i^{(k)}&=P^{(k)\top}w_i^{(k)},&T^{(k)}&=(C^{(k)})^{1/2}P^{(k)}.
\end{aligned}
\end{equation}
Here $T^{(0)}=(C^{(0)})^{1/2}$. We now specialize to Hera-PSD with $W=I_d$, as in the theorem. For $r^{(\ell)}=\|B^{(\ell)}\|_F$, the raw reconstruction is
\begin{equation}\label{eq:v8-physical}
U^{(\ell)}=\frac{r^{(\ell)}}\gamma\tilde U^{(\ell)}(T^{(\ell-1)})^{-1},\qquad
V^{(\ell)}=\frac\gamma{r^{(\ell)}}(T^{(\ell-1)})^{-\top}B^{(\ell)}(T^{(\ell-1)})^{-1}.
\end{equation}
Write $\mathsf B^{(\ell)}=B^{(\ell)}/r^{(\ell)}$. All attention layers use \eqref{eq:multi-parameterization}. The centered softmax statistics and uncentered residual skip are
\begin{equation}\label{eq:attention}
\begin{aligned}
\pi_{ij}^{(\ell)}&=\frac{\exp(\gamma\tilde\xi_i^{(\ell-1)\top}\mathsf B^{(\ell)}\tilde\xi_j^{(\ell-1)})}{\sum_a\exp(\gamma\tilde\xi_i^{(\ell-1)\top}\mathsf B^{(\ell)}\tilde\xi_a^{(\ell-1)})},\\
\mu_i^{(\ell)}&=\sum_j\pi_{ij}^{(\ell)}(\xi_j^{(\ell-1)}-\bar\xi^{(\ell-1)}),\\
\Sigma_i^{(\ell)}&=\sum_j\pi_{ij}^{(\ell)}(\xi_j^{(\ell-1)}-\bar\xi^{(\ell-1)}-\mu_i^{(\ell)})(\xi_j^{(\ell-1)}-\bar\xi^{(\ell-1)}-\mu_i^{(\ell)})^\top.
\end{aligned}
\end{equation}
Thus the actual empirical forward pass is
\begin{equation}\label{eq:v8-forward-normalized}
\xi_i^{(0)}=x_i,\qquad \xi_i^{(\ell)}=\xi_i^{(\ell-1)}+U^{(\ell)}\mu_i^{(\ell)},\qquad \hat y_i=\xi_i^{(L)}.
\end{equation}
The normalized mean equals $(T^{(\ell-1)})^{-1}\mu_i^{(\ell)}=\sum_j\pi_{ij}^{(\ell)}\tilde\xi_j^{(\ell-1)}$. Every operation is differentiated, including $r^{(\ell)}$ and all subsequent coordinate maps. No whitening follows the final attention output.

\subsection{Hera and Hera-PSD as preconditioning}\label{subsec:hera-preconditioning}\label{sec:hera-raw-updates}
Fix a layer $\ell\in\{1,\ldots,L\}$ and an iteration $t$. Write $C=C_t^{(\ell-1)}\succ0$ for its centered empirical input covariance and $T=T_t^{(\ell-1)}=C^{1/2}P_t^{(\ell-1)}$ for its incoming coordinate map, as in \eqref{eq:v8-normalization}. Since $P_t^{(\ell-1)}$ is orthogonal, $TT^\top=C$. Let $\gamma^{(\ell)}>0$ be the fixed score scale and $r_t^{(\ell)}=\|\tilde V_t^{(\ell)}\|_F$. Inverting \eqref{eq:repar-lifted-variables} gives
\[
U_t^{(\ell)}=\frac{r_t^{(\ell)}}{\gamma^{(\ell)}}\tilde U_t^{(\ell)}T^{-1},\qquad
V_t^{(\ell)}=\frac{\gamma^{(\ell)}}{r_t^{(\ell)}}T^{-\top}\tilde V_t^{(\ell)}T^{-1}.
\]
\needspace{6\baselineskip}
Differentiating the loss with respect to the two independent lifted factors gives the gradients used directly by Hera. For Hera-PSD, we then apply the chain rule through the tie in \eqref{eq:multi-parameterization}. Write $\widehat{\mathcal L}(\theta)$ for the full multilayer empirical loss and define
\[
G_{U,t}^{(\ell)}=\nabla_{U^{(\ell)}}\widehat{\mathcal L}(\theta_t),\qquad
G_{V,t}^{(\ell)}=\nabla_{V^{(\ell)}}\widehat{\mathcal L}(\theta_t),\qquad
\alpha_t^{(\ell)}=\langle U_t^{(\ell)},G_{U,t}^{(\ell)}\rangle_F-\langle V_t^{(\ell)},G_{V,t}^{(\ell)}\rangle_F,
\]
where $\langle A,B\rangle_F=\Tr(A^\top B)$. These raw-coordinate derivatives hold the other raw weight matrices fixed and already backpropagate the loss through all subsequent attention layers.

\par\smallskip\noindent\textbf{Additional gradients through subsequent layers.}
In lifted coordinates, changing a factor in layer $\ell$ also changes the input covariances and polar alignments used to reconstruct the raw weights in layers $k>\ell$, even when their lifted factors are fixed. Let $E^\ell_{U,t}$ and $E^\ell_{V,t}$ denote these additional loss-gradient contributions, expressed in the coordinates specified below. More precisely, with $\tilde G_{U,t}^{(\ell)}=\nabla_{\tilde U^{(\ell)}}\widehat{\mathcal L}(\tilde\theta_t)$ and $\tilde G_{V,t}^{(\ell)}=\nabla_{\tilde V^{(\ell)}}\widehat{\mathcal L}(\tilde\theta_t)$ denoting the full lifted loss gradients, they are defined by
\begin{equation}\label{eq:hera-multilayer-gradient-residuals}
\begin{aligned}
\tilde G_{U,t}^{(\ell)}
&=\frac{r_t^{(\ell)}}{\gamma^{(\ell)}}G_{U,t}^{(\ell)}T^{-\top}
+\frac{\gamma^{(\ell)}}{r_t^{(\ell)}}E^\ell_{U,t}T,\\
\tilde G_{V,t}^{(\ell)}
&=\frac{\gamma^{(\ell)}}{r_t^{(\ell)}}T^{-1}G_{V,t}^{(\ell)}T^{-\top}
+\frac{\alpha_t^{(\ell)}}{(r_t^{(\ell)})^2}\tilde V_t^{(\ell)}
+\frac{r_t^{(\ell)}}{\gamma^{(\ell)}}T^\top E^\ell_{V,t}T.
\end{aligned}
\end{equation}
Thus $E^\ell_{U,t}$ and $E^\ell_{V,t}$ contain the remaining chain-rule terms through the later raw-weight reconstructions, rather than a second copy of ordinary backpropagation. They are exact derivatives, not approximation errors or regularization terms. They vanish in the last attention layer because no later raw attention weights are reconstructed, and in particular when $L=1$. For Hera-PSD, the full loss gradients in the factor coordinates of one layer are
\begin{equation}\label{eq:hera-psd-gradients}
\begin{aligned}
\nabla_B\widehat{\mathcal L}&=\tilde G_V+\tilde G_U S(K),\\
\nabla_K\widehat{\mathcal L}&=-K S(K)(B^\top\tilde G_U+\tilde G_U^\top B)S(K).
\end{aligned}
\end{equation}
Here $dS=-S(K^\top dK+dK^\top K)S$. Adding $\nabla_B\mathcal R$ and $\nabla_K\mathcal R$ to the respective gradients in \eqref{eq:hera-psd-gradients} gives the two directions in Algorithm~\ref{alg:hera-psd}.

\paragraph{Reconstruction after a step.}
For the Hera-PSD step, let $\Delta B=B_{t+1}-B_t$ in this layer.
Holding the incoming map $T$ fixed momentarily gives the provisional raw weights
\begin{equation}\label{eq:hera-raw-U-update}
\bar U_{t+1}=\frac{r_{t+1}}{r_t}U_t+
  \frac{r_{t+1}}\gamma[B_{t+1}S(K_{t+1})-B_t S(K_t)]T^{-1},
\end{equation}
\begin{equation}\label{eq:hera-raw-V-update}
\bar V_{t+1}=\frac{r_t}{r_{t+1}}V_t+
   \frac{\gamma}{r_{t+1}}T^{-\top}\Delta B\,T^{-1},\qquad r_{t+1}=\|B_{t+1}\|_F .
\end{equation}
The next forward pass updates the preceding layers as well. If $M_t=T_t^{(\ell-1)}(T_{t+1}^{(\ell-1)})^{-1}$, the actual raw weights from \eqref{eq:hera-raw-U-update} and \eqref{eq:hera-raw-V-update} are $U_{t+1}=\bar U_{t+1}M_t$ and $V_{t+1}=M_t^\top\bar V_{t+1}M_t$.

For a differential interpretation, let $J_x$ be the Jacobian of the full empirical reconstruction from the chosen optimization coordinates to the raw weights. If $\mathsf G$ is the gradient metric at the current iterate, the induced first-order raw direction is
\[
-J_x\mathsf G^{-1}J_x^\top\nabla_\theta\widehat{\mathcal L}
-J_x\mathsf G^{-1}\nabla\mathcal R.
\]
For the Hera-PSD step, $\mathsf G^{-1}$ multiplies the $B$-blocks by $\eta_B$ and the $K$-blocks by $\eta_K$.

\section{Additional related work}\label{sec:additional-related-work}

\subsection{Matrix factorization}
We show that the infinite-sample population loss is a matrix factorization loss; algorithms for matrix factorization can hence be adapted for self-attention training. Procrustes Flow \citep{tu16} achieves geometric convergence for low-rank matrix sensing under a restricted isometry condition and suitable initialization; its rectangular formulation uses a penalty on the difference between the factors' Gram matrices. Ge, Jin, and Zheng \citep{ge2017no} establish the absence of spurious local minima and a strict-saddle landscape for matrix sensing, matrix completion, and robust PCA. Du, Hu, and Lee \citep{du2018algorithmic} establish conservation of Gram-matrix differences under gradient flow in linear networks and analyze implicit balancing in unregularized matrix factorization. Of particular relevance to our work is the work of Tarmoun, Fran\c{c}a, Haeffele, and Vidal \citep{tarmoun2021understanding}, who study the optimization dynamics of matrix factorization with an explicit penalty toward a prescribed symmetric Gram-matrix difference $Q_0$. We use a similar penalty in our analysis of Hera-PSD.

\subsection{Structure-aware optimization}
A recent line of ``structure-aware" optimization algorithms (e.g., Muon  \cite{jordan6muon}, Shampoo  \cite{gupta2018shampoo}, Soap \cite{vyas2024soap}) use preconditioning to update the self-attention parameters and generally outperform classical algorithms such as Adam.  We refer to \cite{lau2025polargrad} for a unified perspective on such preconditioned algorithms. Our algorithms can also be interpreted as preconditioned gradient descent methods; we refer to Appendix \ref{subsec:hera-preconditioning} for more details.

\subsection{Norm-constrained optimization}
A key aspect of our Hera and Hera-PSD algorithms is that we fix the norm of the ke-query matrix in lifted coordinates. A series of recent papers also study optimization in such norm-constrained settings.
Wen et al.~\citep{wen2026hyperball} introduce Hyperball, a wrapper that normalizes the update supplied by a base optimizer such as Adam or Muon and returns each updated matrix to a prescribed Frobenius sphere. They report improvements in language-model pretraining efficiency and learning-rate transfer across model widths and depths. Xie et al.~\citep{xie2026spectral} introduce the Spectral Sphere Optimizer, which controls the spectral norms of weight matrices  during language-model pretraining, and show that this algorithm outperforms Muon and AdamW across many model architectures, including models with 200 layers. Another line of work separates each weight into a direction and a norm, and adjusts each of these components separately. Salimans and Kingma~\citep{salimans2016weight} reparameterize each weight vector by a direction and a separately learned magnitude to improve the conditioning of stochastic optimization. Heo et al.~\citep{heo2021adamp} show that momentum can accelerate parameter growth and thereby reduce the effective learning rate prematurely. Their SGDP and AdamP methods remove the radial component of an optimizer update for scale-invariant weights, slowing this growth without imposing an exact norm constraint.



\subsection{Depth scaling and trainability}\label{sec:related-depth-scaling}
A goal of our work is to show that our convergence analysis extends to models of arbitrary depth. There has been much prior work in understanding optimization in models with many layers.
Hayou et al.~\citep{hayou2021stable} propose the Stable ResNet architecture, which ensures training stability while preserving expressivity as depth increases. A key aspect of this architecture is depth-dependent scaling of the nonlinear activations appearing in each layer. Wang et al.~\citep{wang2024deepnet} introduce DeepNet, using the DeepNorm residual connection together with a depth-dependent initialization to stabilize very deep transformers. Their experiments demonstrate successful training of transformers with up to 1,000 layers.  Yang et al.~\citep{yang2024depthmup}  propose the Depth-$\mu$P framework for hyperparameter transfer in large depth Transformers, and identify a specific scaling regime in which hyperparameter transfer across depth is possible. Bordelon, Chaudhry, and Pehlevan~\citep{bordelon2024infinite} analyze multi-head transformer training using dynamical mean-field theory and identify parameterizations with well-defined large-width and large-depth limits in which attention continues to learn.



\section{Proof of the multilayer convergence theorem}\label{app:frozen-gaussian}\label{sec:multi}
This appendix proves Theorem~\ref{thm:frozen-gaussian}. All forward passes use the original centered, unscaled, lifted softmax model, with differentiation through every empirical covariance and polar alignment.
The data obey $p=d$, $W=I_d$ is fixed, $\Sigma\succ0$, and $X_\star=(\Phi-I_d)\Sigma^{1/2}$ is symmetric positive definite. Write $\theta_t=\{U_t^{(\ell)},V_t^{(\ell)}\}_{\ell=1}^L$ for the raw weights.

\subsection{Factor updates and Gram balance}\label{sec:multi-psd}
Abbreviate the trainable factors of layer $\ell$ by $B_\ell,K_\ell$. Write $S_\ell=(I_d+K_\ell^\top K_\ell)^{-1}$, $A_\ell=B_\ell S_\ell B_\ell^\top$, $P=\sum_{\ell=1}^L A_\ell$, and $P_{\rm h}=\sum_{\ell<L}A_\ell$. Fix $m=\sigma_{\min}(\Sigma^{1/2})$, $M=\|\Sigma^{1/2}\|_{\op}$, $\nu=\sqrt{\Tr\Omega}$ and class bounds $0<\mu\le\lambda_{\min}(X_\star)$ and $\rho\ge\max\{1,\|X_\star\|_F\}$. For $L>1$, set
\[
 s=\min\{1,\mu/18,m/(9216\sqrt d)\},\qquad r_{\rm h}=\sqrt{s/(L-1)} .
\]
For $L=1$ set $s=r_{\rm h}=0$ and omit all expressions indexed by $\ell<L$. Then $P_{\rm h}=0$ and the last layer fits $X_\star$ directly. The fixed $r_{\rm h}$ sets the scale of the earlier factors in the penalty and initialization. The forward pass sums the increments $A_\ell$ with unit coefficients.

Define the Gram residuals $D_\ell=r_{\rm h}^{-2}B_\ell^\top B_\ell-4I_d-S_\ell$ for $\ell<L$, and $D_L=K_L^\top K_L-I_d-\beta^{-1}B_L^\top B_L$. With positive strengths $\kappa,\beta$ specified below, use
\begin{equation}\label{eq:fg-reg}
 \mathcal R=\frac{\kappa r_{\rm h}^2}{4}\sum_{\ell<L}\|D_\ell\|_F^2+\frac{\beta}{4}\|D_L\|_F^2 .
\end{equation}
Given $\widehat{\mathcal Q}=\widehat{\mathcal L}+\mathcal R$, Algorithm~\ref{alg:hera-psd} evaluates every gradient at the same factors and takes the additive steps
\begin{align}
 B_{\ell,t+1}&=B_{\ell,t}-\eta_B\nabla_{B_\ell}\widehat{\mathcal Q}_t,\label{eq:multi-update}\\
 K_{\ell,t+1}&=K_{\ell,t}-\eta_K\nabla_{K_\ell}\widehat{\mathcal Q}_t.\label{eq:fg-add}
\end{align}
In the auxiliary Gaussian model the regularized excess objective is
\begin{equation}\label{eq:fg-quad}
 f=\tfrac12\|P-X_\star\|_F^2+\mathcal R .
\end{equation}
Every balanced choice of earlier gates and orientations extends to an exact fit. Indeed, for arbitrary $K_\ell$ and orthogonal $O_\ell$, take $B_\ell=r_{\rm h}O_\ell(4I_d+S_\ell)^{1/2}$. Then $D_\ell=0$ and $0\preceq P_{\rm h}\preceq5sI_d$. The remaining target $Y=X_\star-P_{\rm h}$ is positive definite. Since $\beta>\rho$, one can choose
\begin{equation}\label{eq:fg-chart}
 B_L B_L^\top=2Y(I_d-Y/\beta)^{-1},\qquad
 K_L^\top K_L=I_d+B_L^\top B_L/\beta .
\end{equation}
Indeed, with $C=B_LB_L^\top$, the last increment in \eqref{eq:fg-chart} is $C(2I_d+C/\beta)^{-1}=Y$. Hence the minimum of \eqref{eq:fg-quad} is zero.

\subsection{Initialization, step sizes and detailed guarantee}\label{sec:multi-theorem}
Let $a=\sqrt{\mu/2}$, $b=\sqrt{8\rho}$, $h=(a/b)^3/128$ and $T_*=b^2+9s$. Use the following fixed constants:
\[
 \kappa=64(1+T_*+\rho),\qquad \beta=16b^4/h,\qquad H=200(1+T_*+\rho)^2+100\kappa .
\]
Choose the two learning rates and define the contraction exponent by
\begin{equation}\label{eq:direct-rates}
 \eta_B=\min\{1/(2H),h/\rho\},\qquad
 \eta_K=\min\{\eta_B,1/(32\beta)\},\qquad
 c=\min\{\mu\eta_B/64,\beta\eta_K/8\}.
\end{equation}
Independently of the sample, draw independent standard Gaussian $d\times d$ matrices for each $G_\ell$ and for each polar factor $O_\ell,O'_L$ below (each polar factor is taken from a fresh draw). Initialize by
\begin{equation}\label{eq:multi-initialization}
\begin{aligned}
 K_{\ell,0}&=G_\ell,\quad B_{\ell,0}=r_{\rm h}O_\ell(4I_d+S(K_{\ell,0}))^{1/2} &&(\ell<L),\\
 B_{L,0}&=aO_L,\quad K_{L,0}=O'_L(I_d+B_{L,0}^\top B_{L,0}/\beta)^{1/2}. &&
\end{aligned}
\end{equation}
All Gram residuals vanish initially. Since $S(K_{L,0})\preceq I_d/2$, we have $0\preceq P_0\preceq(5s+a^2/2)I_d\preceq X_\star$ and $2f_0=\|P_0-X_\star\|_F^2\le\|X_\star\|_F^2$.
\par

Set
\begin{equation}\label{eq:choices}
q=2d^2\log(2n)+\log(16/\delta),\quad R_n=\sqrt d+\sqrt{2\log n}+\sqrt{8q+2},\quad \epsilon=48d\sqrt{q/n}.
\end{equation}
Use the fixed normalized score scale
\begin{equation}\label{eq:fg-score}
0<\gamma\le\min\{3/(16\sqrt q R_n),1/(64R_n^2)\}.
\end{equation}
Set $J=\sqrt d\,T_*$, $B_T=(M+m)\sqrt d$ and $C_T=2B_T(1+10J/m)$. For a directional bound $D\ge0$ define
\[
 C(D)=D\{\rho M/(3m)+\nu+3J+(2\rho+2\nu+4J)(16+200J/m)\},\qquad
 G=C(1)\big[(1+\sqrt d)\sqrt{T_*}+\sqrt d\,T_*\big].
\]
Require $n\ge128$, $n\ge100q^3$ and
\begin{equation}\label{eq:multi-sample-conditions}
 \epsilon\le\min\{1/16,\ m/[256(M+m)],\ m/(8B_T),\ 1/G,\ a^3/(32G)\}.
\end{equation}
Write $\sigma(B)\subset[u,v]$ to mean that every singular value of $B$ lies in $[u,v]$. Define the following region in factor space:
\begin{equation}\label{eq:direct-region}
 \sigma(B_\ell/r_{\rm h})\subset[1,3]\quad(\ell<L),\qquad
 \sigma(B_L)\subset[a,b],\qquad \sigma(K_L)\subset[1-h,1+h].
\end{equation}
The earlier gates $K_\ell$ can be arbitrary in this region.

\begin{theorem}[Detailed multilayer statement]\label{thm:multi-empirical-detailed}
Under the data assumptions above, use \eqref{eq:fg-reg}, initialization \eqref{eq:multi-initialization}, rates \eqref{eq:direct-rates}, score scale \eqref{eq:fg-score}, and sample conditions \eqref{eq:multi-sample-conditions}. With probability at least $1-\delta$, the empirical iterates of \eqref{eq:multi-update} and \eqref{eq:fg-add} stay in \eqref{eq:direct-region} for every $t\ge0$. All whitening and alignment operations are defined, and the corresponding raw parameters satisfy
\[
 \mathcal L(\theta_t)-\mathcal L^\star
 \le \|X_\star\|_F^2\exp(-ct)+K_{\rm risk}\,d^2q/n,
\]
where
\begin{equation}\label{eq:multi-risk-constant}
 K_{\rm risk}=48^2(C_T^2+\eta_B G^2/c).
\end{equation}
\end{theorem}
For $L>1$ all constants above are independent of depth. The same bounds work for $L=1$ if the positive upper bound on $s$ from the multilayer case is used to evaluate them. Their dependence on the population scales and their indicated inverses is polynomial, up to logarithms in the sample threshold.
\par

\subsection{Uniform comparison of empirical gradients and population risk}\label{sec:single-concentration}
Work on the single event in Lemma~\ref{lem:v8-data-event}. We verify the deterministic hypotheses of the token estimates uniformly on \eqref{eq:direct-region}. Put $b_\ell=3r_{\rm h}$ for $\ell<L$ and $b_L=b$, so $\sum_\ell b_\ell^2\le T_*$. Since $0\prec S_\ell\preceq I_d$,
\[
 k'_\ell=\|B_\ell\|_F\|\tilde U_\ell\|_{\op}\le\sqrt d\,b_\ell^2,\qquad
 \sum_{\ell<L} k'_\ell\le 9\sqrt d\,s\le m/1024,\qquad \sum_\ell k'_\ell\le J .
\]
For a factor variation let $d'_\ell=\|B_\ell\|_F\|\dot{\tilde U}_\ell\|_F+\|\tilde U_\ell\|_{\op}\|\dot B_\ell\|_F$. The gate derivative is $\dot S=-S(\dot K^\top K+K^\top\dot K)S$. The inequalities $\|S\|_{\op}\le1$ and $\|KS\|_{\op}\le1/2$ give $\|\dot S\|_F\le\|\dot K\|_F$ for every $K$. Thus
\[
 d'_\ell\le(1+\sqrt d)b_\ell\|\dot B_\ell\|_F+\sqrt d\,b_\ell^2\|\dot K_\ell\|_F .
\]
\par

\paragraph{Propagation through empirical whitening.}
For clarity we give the constants used from the token estimates. Token arrays have columns divided by $\sqrt n$; for $D=[d_i]/\sqrt n$ use $\|D\|_*=\max(\|D\|_F,\max_i\|d_i\|/R_n)$. Write $H_j^w$ for the centered white aligned token array, $T_j$ for its raw coordinate map, and $S^c_j=C_0^{1/2}+\sum_{i\le j}A_i$ for the empirical linear prefix. Here $C_0$ is the centered input covariance and $T_0=C_0^{1/2}$. The three existing inputs to this argument are the shared original-sample Pisier event, the softmax stability lemma and the whitening and polar-alignment differential%
\ifdefined\standaloneGaussianProof{} (respectively \texttt{lem:v8-data-event}, \texttt{lem:multi-stability}, and \texttt{lem:multi-normalization} in the manuscript)\else{} (Lemmas~\ref{lem:v8-data-event}, \ref{lem:multi-stability} and~\ref{lem:multi-normalization})\fi.
The first bounds the normalized softmax mean remainder (after division by $\gamma$) and covariance discrepancy by $\epsilon$ at the original white array, uniformly over score matrices of unit Frobenius norm. It bounds $\|C_0^{1/2}-\Sigma^{1/2}\|_{\op}$ by $M\epsilon/3$, the Frobenius norm of the noise--white-input cross moment by $\nu\epsilon$, and the initial maximum radius by $3R_n/4$.

For $\chi=7\gamma R_n^2\le7/64$, let $\delta_j=\|H_j^w-H_0^w\|_*$, $k_{\rm pre}=\sum_{j<L}k'_j$. Subtract the linear part $(T_{j-1}+A_j)H_{j-1}^w$ of the centered output; call the remainder $N_j$. It is the centered normalized mean remainder multiplied by $\|B_j\|_F\tilde U_j$. The stability lemma gives $\|N_j\|_*\le2k_j'(\epsilon+\chi\delta_{j-1})$. The whitening lemma then bounds $\delta_j$ by $\delta_{j-1}+10 k_j'(\epsilon+\chi\delta_{j-1})/m$. Iteration through $j<L$ gives $\delta_j\le12 k_{\rm pre}\epsilon/m\le\epsilon$, $\|N_j\|_*\le3k'_j\epsilon$, and
$\|T_j-S^c_j\|_F\le6k_{\rm pre}\epsilon$. For the latter bound the change in the fitted coordinate map is at most $\|N_j\|_*+4(M+m)\|N_j\|_*^2/m^2$ beyond the linear increment. To check the induction domain, $S^c_j$ has least singular value at least $\sqrt{1-\epsilon/3}\,m$ and norm at most $\sqrt{1+\epsilon/3}\,M+k_{\rm pre}$. The preceding coordinate error and current fit error together are at most $9k_{\rm pre}\epsilon$, so the fitted map has singular values in $[m/2,2(M+m)]$. The two whitening smallness checks $2\|N_j\|_*/m\le1/64$ and $8(M+m)\|N_j\|_*/m^2\le1/16$ follow from the sample conditions. All input radii stay within $R_n$, and $\|T_j^{-1}\|_{\op}\le2/m$ before every attention layer.

The covariance error from the stability lemma is at most $5\epsilon$. Thus the derivative of the centered or uncentered mean remainder after multiplication by the factors has norm at most $\chi k'_j\|\dot H_{j-1}^w\|_F+10\epsilon d'_j$: vary the tokens using stability, and the score at fixed tokens using the covariance discrepancy times the varied score vector. The radial derivative is included by $(\dot r)^2+r^2\|\partial(B/r)\|_F^2=\|\dot B\|_F^2$, $r=\|B\|_F$. With $v_j'=\|\dot H_j^w\|_F+2\|\dot T_j-\dot S^c_j\|_F/m$, the whitening differential gives
\[
 v'_j\le[1+(6\chi+24\epsilon)k'_j/m]v'_{j-1}
       +60\epsilon d'_j/m+(24k'_j\epsilon/m^2)\sum_{i\le j}d'_i .
\]
Since $v'_0=0$ and $k_{\rm pre}\le m/1024$, this yields $v'_j\le100\epsilon\sum_i d'_i/m$ for $j<L$.

Writing $X$ for the raw input array, let $\Xi_0=X+PH_0^w$ and $\Xi$ be the actual uncentered final array. The output difference is the sum of $A_j(H_{j-1}^w-H_0^w)$ and the uncentered mean remainders. Consequently its Frobenius norm is at most $3J\epsilon$ and its directional derivative has norm at most $(16+200J/m)\epsilon\sum d'_i$. The empirical linearized residual has norm at most $2\rho+2\nu+J$ by the sample event, and $\|\dot\Xi_0\|_F\le\sum d'_i$. Expanding the squared-loss derivative yields the last two terms in $C(D)$. Up to a parameter-independent constant, the linearized empirical loss is $\tfrac12\|P-(\Phi-I_d)C_0^{1/2}-N_z\|_F^2$ with noise cross moment $N_z$. Since $\|\Phi-I_d\|_F=\|X_\star\Sigma^{-1/2}\|_F\le\rho/m$, the empirical residual target differs from $X_\star$ in Frobenius norm by at most $(\rho M/(3m)+\nu)\epsilon$. Pairing this error with the variation of $P$ gives the first two terms in $C(D)$.
The same regularizer appears in both objectives, so the directional gradient discrepancy is bounded by $C(1)\epsilon\sum_\ell d'_\ell$. Cauchy--Schwarz and the preceding factor-variation estimate give, in the product Frobenius norm,
\begin{equation}\label{eq:fg-grad}
 \|\nabla\widehat{\mathcal Q}-\nabla f\|_F\le G\epsilon,\qquad
 \|\nabla_{B_\ell}(\widehat{\mathcal Q}-f)\|_F\le 3(1+\sqrt d)C(1)r_{\rm h}\epsilon\quad(\ell<L).
\end{equation}
The second bound follows by varying only $B_\ell$. Its factor $r_{\rm h}$ is needed when checking a single earlier layer. Since $T_*\ge8$, $3(1+\sqrt d)C(1)\le G$.
\par

\paragraph{Population risk at the empirical raw weights.}\label{sec:multi-risk-transfer}
Let $F_jg$ be the actual population token for $g=\Sigma^{-1/2}x$, and $S^p_j=\Sigma^{1/2}+\sum_{i\le j}A_i$. The preceding comparison gives $\|T_j-S^p_j\|_F\le B_T\epsilon$ for $j<L$. Empirical reconstruction in the true population covariance recursion gives $F_j=F_{j-1}+\tilde U_j Q Q^\top B_j^\top Q$, where $Q=T_{j-1}^{-1}F_{j-1}$ and $F_0=\Sigma^{1/2}$. Put $e_j=\|F_j-S^p_j\|_F$, $p_j=\|\tilde U_j\|_{\op}\|B_j\|_{\op}$. For $\|Q-I_d\|_{\op}\le1/2$, telescoping the three occurrences of $Q$ gives
\[
 e_j+B_T\epsilon\le(1+10p_j/m)(e_{j-1}+B_T\epsilon).
\]
Indeed the error in the increment is at most $5p_j\|Q-I_d\|_F$ and $\|T_{j-1}^{-1}\|_{\op}\le2/m$. The prefix amplification is at most $\exp(10/1024)<2$ since $p_j\le k'_j$. Thus $\|Q-I_d\|_{\op}\le4B_T\epsilon/m\le1/2$ before every step. Using $p_L\le J$ once gives
\begin{equation}\label{eq:fg-raw}
 \sqrt{2(\mathcal L(\theta)-\mathcal L^\star)}
 =\|F_L-\Phi\Sigma^{1/2}\|_F\le \|P-X_\star\|_F+C_T\epsilon .
\end{equation}

\subsection{Proof of the all-time bound for factor GD}\label{sec:multi-convergence}
\begin{proof}[Proof of Theorem~\ref{thm:multi-empirical-detailed}]
We first verify \eqref{eq:direct-region} by induction on the empirical iterates. Initialization \eqref{eq:multi-initialization} satisfies it: the earlier scaled factors have singular values in $[2,\sqrt5]$, and $\sqrt{1+a^2/\beta}-1\le a^2/(2\beta)<h$. Suppose the region holds at time $t$. Write $E=P-X_\star$. Then $\|E\|_{\op}\le T_*+\rho$, all gradients exist by the whitening estimates, and \eqref{eq:fg-grad} applies at this iterate.

\paragraph{Earlier factors.}
Fix $\ell<L$ and put $D=B_\ell/r_{\rm h}$. Differentiating \eqref{eq:fg-quad} gives $\nabla_{B_\ell} f=2E B_\ell S_\ell+\kappa B_\ell D_\ell$. Hence the empirical $D^+=B_{\ell,t+1}/r_{\rm h}$ differs from $D[I_d+\eta_B\kappa(4.5I_d-D^\top D)]$ in operator norm by at most
\[
 \eta_B\big[1.5\kappa+6(T_*+\rho)+G\epsilon\big]\le 2\eta_B\kappa.
\]
Here $\|D(S_\ell-I_d/2)\|_{\op}\le 1.5$ for every gate. The comparison matrix sends a singular value $x$ of $D$ to $x[1+\eta_B\kappa(4.5-x^2)]$. This function is increasing on $[1,3]$ because $\eta_B\kappa\le1/200$. At the endpoints its values are $1+3.5\eta_B\kappa$ and $3-13.5\eta_B\kappa$. The perturbation bound therefore gives $\sigma(D^+)\subset[1,3]$, independently of the simultaneous step on $K_\ell$.

\paragraph{Last gate.}
Write $B=B_L$, $K=K_L$, $S=S_L$. The two auxiliary gradients in this layer are
\[
 \nabla_B f=2EBS-BD_L,\qquad
 \nabla_K f=-2KS B^\top E B S+\beta K D_L .
\]
Since $0\preceq P_{\rm h}\preceq9sI_d$, the matrix $Y=X_\star-P_{\rm h}$ satisfies $a^2I_d\preceq Y\preceq\rho I_d$ and $\|E\|_{\op}\le b^2+\rho$. The comparison step $K-\eta_K\beta K(K^\top K-I_d)$ sends a singular value $x$ to $g(x)=x-\eta_K\beta x(x^2-1)$. On $[1-h,1+h]$, $g(x)-1=(x-1)[1-\eta_K\beta x(x+1)]$, so $|g(x)-1|\le(1-\eta_K\beta)h$. The remaining terms in the empirical step have operator norm at most
\[
 \eta_K\big[b^2(b^2+\rho)+2b^2+G\epsilon\big]\le \eta_K\beta h .
\]
This uses $\|KS\|_{\op}\le1/2$, $\|K\|_{\op}\le2$ and $\beta h=16b^4$. Singular-value perturbation gives $\sigma(K^+)\subset[1-h,1+h]$.

\paragraph{Last value factor.}
At the old iterate, $\|S-I_d/2\|_{\op}\le h$ and $\|D_L\|_{\op}\le3h+b^2/\beta$. Compare the empirical $B^+$ with
\[
 B_{\rm cmp}=(I_d+\eta_B Y)B(I_d-\eta_B B^\top B/2).
\]
Replacing $S$ by $I_d/2$ in $2(BSB^\top-Y)BS$ incurs operator-norm error at most $(3b^2+2\rho)bh$. The balance gradient costs at most $b(3h+b^2/\beta)$. Expanding the product defining $B_{\rm cmp}$ and using \eqref{eq:fg-grad} now gives
\[
 \frac{\|B^+-B_{\rm cmp}\|_{\op}}{\eta_B}
 \le (3b^2+2\rho)bh+b(3h+b^2/\beta)+G\epsilon+\eta_B\rho b^3/2
 \le a^3/8 .
\]
For the last inequality use $b^2=8\rho\ge8$, $b^3h=a^3/128$, $\eta_B\rho\le h$ and \eqref{eq:multi-sample-conditions}. The singular values of the right product $B(I_d-\eta_B B^\top B/2)$ are $x(1-\eta_B x^2/2)$, increasing on $[a,b]$. Thus the least singular value of $B_{\rm cmp}$ is at least $(1+\eta_B a^2)a(1-\eta_B a^2/2)\ge a+\eta_B a^3/4$. Its greatest singular value is at most $(1+\eta_B\rho)b(1-\eta_B b^2/2)\le b-2\eta_B\rho b$. Both margins absorb the perturbation. This proves \eqref{eq:direct-region} at the next iterate and closes the induction. In particular $\|B_{\ell,t}\|_F>0$ wherever the forward pass divides by this norm.

\paragraph{Gradient domination in the factors.}
Put $e=\|E\|_F$, $v=\|D_L\|_F$ and $z^2=r_{\rm h}^2\sum_{\ell<L}\|D_\ell\|_F^2$ (zero in one layer). Let $g_{B,L},g_{K,L}$ be the Frobenius norms of the corresponding auxiliary gradients, and $g_{B,{\rm h}}$ the product norm of the earlier $B$-gradients. The last-layer formulas above, $\sigma_{\min}(B)\ge a$, $\sigma_{\min}(K)\ge1/2$ and $\sigma_{\min}(S)\ge1/3$ give
\[
 g_{B,L}\ge(2a/3)e-bv,\qquad g_{K,L}\ge(\beta/2)v-b^2e .
\]
Likewise $\kappa z\le g_{B,{\rm h}}+6\sqrt{s}\,e$ by the earlier $B$-gradient formula and $\sigma_{\min}(B_\ell)\ge r_{\rm h}$. The last-layer inequalities imply $e\le(2/a)(g_{B,L}+2b g_{K,L}/\beta)$ and $v\le(2/\beta)(g_{K,L}+b^2 e)$, since $\beta\ge12b^3/a$. Inserting these bounds in $f=e^2/2+\beta v^2/4+\kappa z^2/4$ gives
\begin{equation}\label{eq:direct-pl}
 f\le(32/\mu)\|\nabla_B f\|_F^2+(4/\beta)\|\nabla_K f\|_F^2 .
\end{equation}
Indeed, before substituting for $e$, its squared coefficient is at most $1/2+2b^4/\beta+18s/\kappa\le1$. The remaining coefficients are $2/\beta$ on $g_{K,L}^2$ and $1/(2\kappa)$ on $g_{B,{\rm h}}^2$; use $\beta\ge32b^2/a^2$ and $a^2=\mu/2$.

\paragraph{Smoothness for the two constant steps.}
Along the straight segment between consecutive iterates the upper operator-norm bounds on $B_\ell$ and $K_L$ persist by convexity. We bound the second derivative of $f$ in any direction $(U,V)$ of variations of $(B,K)$, with $\|(U,V)\|^2=\sum_\ell(\|U_\ell\|_F^2+\|V_\ell\|_F^2)$. Denote the matrix nuclear norm by $\|\cdot\|_{\rm nuc}$. For a single gate set $Z=S^{1/2}(K^\top V+V^\top K)S^{1/2}$. Since $\|KS^{1/2}\|_{\op}\le1$, $\|Z\|_F\le2\|V\|_F$. Differentiating the inverse twice gives $\ddot S=2S^{1/2}Z^2 S^{1/2}-2S V^\top V S$, hence $\|\ddot S\|_{\rm nuc}\le10\|V\|_F^2$ for every $K$. The product rule for $P$ yields
\[
 \|\dot P\|_F\le2\sqrt{T_*}\|U\|_F+T_*\|V\|_F,\qquad
 \|\ddot P\|_{\rm nuc}\le(2+2\sqrt{T_*}+10T_*)\|(U,V)\|^2 .
\]
Together with $\|E\|_{\op}\le T_*+\rho$, these bound the Hessian of the prediction term. For each earlier penalty, $r_{\rm h}\|\dot D_\ell\|_F\le6\|U_\ell\|_F+r_{\rm h}\|V_\ell\|_F$ and $r_{\rm h}^2\|\ddot D_\ell\|_{\rm nuc}\le2\|U_\ell\|_F^2+10r_{\rm h}^2\|V_\ell\|_F^2$. As $\|D_\ell\|_{\op}\le14$ and $r_{\rm h}\le1$, the sum of their Hessian forms is bounded above by $90\kappa\|(U,V)\|^2$.
The final penalty expands as
\[
 \tfrac\beta4\|K^\top K-I_d\|_F^2
 -\tfrac12\langle K^\top K-I_d,B^\top B\rangle_F+\tfrac1{4\beta}\|B^\top B\|_F^2 .
\]
Its first term contributes at most $13\beta\|V_L\|_F^2$ since $\|K\|_{\op}\le2$. The remaining two contribute at most $(5+4b+b^2+3b^2/\beta)\|(U,V)\|^2$ by the product rule. Adding the estimates bounds the Hessian form by $H\|(U,V)\|^2+13\beta\|V\|_F^2$, hence by $\|U\|_F^2/\eta_B+\|V\|_F^2/\eta_K$ using \eqref{eq:direct-rates}.

\paragraph{Descent on the common sample event.}
Apply Taylor's inequality on this segment with $U=-\eta_B\nabla_B\widehat{\mathcal Q}$ and $V=-\eta_K\nabla_K\widehat{\mathcal Q}$. Completing the two squares and using \eqref{eq:fg-grad} gives
\[
 f_{t+1}\le f_t-\tfrac12\big(\eta_B\|\nabla_B f_t\|_F^2+\eta_K\|\nabla_K f_t\|_F^2\big)
                 +\tfrac{\eta_B}{2}G^2\epsilon^2
 \le (1-c)f_t+\tfrac{\eta_B}{2}G^2\epsilon^2 ,
\]
where the last inequality uses \eqref{eq:direct-pl}. Iteration gives $f_t\le\exp(-ct)f_0+\eta_B G^2\epsilon^2/(2c)$. Equation~\eqref{eq:fg-raw} and $\|P-X_\star\|_F^2\le2f_t$ imply $\mathcal L(\theta_t)-\mathcal L^\star\le2f_t+C_T^2\epsilon^2$. Substituting the initialization bound and \eqref{eq:choices} gives \eqref{eq:multi-risk} for every $t$ on the same event.
\end{proof}

\section{Supporting lemmas}\label{sec:supporting-lemmas}
We first establish the joint sample event and the softmax statistics at the input, then give the deterministic bounds that propagate these estimates through the network.

\subsection{One uniform Pisier event on the original Gaussian sample}\label{sec:v8-data-event}
\subsubsection{The sample event}
Write $R=R_n$ and $\varepsilon=\epsilon$ for the radius and tolerance in \eqref{eq:choices}. Use $q$ from the same equation, and assume $n\ge128$, $n\ge100q^3$, and $\varepsilon\le1/16$. Put $g_i=\Sigma^{-1/2}x_i$, $\bar g=n^{-1}\sum_i g_i$, and $H^{(0)}=[\tilde\xi_i^{(0)}]_i/\sqrt n$. For any token array, a score vector $b$ assigns weights proportional to $\exp(b^\top \tilde\xi_j)$. Denote its exponentially weighted mean by $\nu_H(b)$ and covariance by $\mathcal C_H(b)$.

\begin{lemma}[Shared original-sample event]\label{lem:v8-data-event}
With probability at least $1-\delta/2$, the following hold simultaneously:
\begin{equation}\label{eq:v8-event-cov}
\begin{gathered}
(1-\varepsilon/3)\Sigma\preceq C^{(0)}\preceq(1+\varepsilon/3)\Sigma,\\
\|\bigl(C^{(0)}\bigr)^{1/2}-\Sigma^{1/2}\|_{\op}\le M\varepsilon/3,
\qquad \|\bar x\|_2\le M\sqrt d\,\zeta,
\quad \zeta=12\sqrt{q/n},\\
\max_i\|\tilde\xi_i^{(0)}\|_2\le3R/4,
\qquad \|[g_i]_i/\sqrt n\|_{\op}\le2,\\
\sup_{\|b\|_2\le3/(16\sqrt q)}
\|\mathcal C_{H^{(0)}}(b)-I_d\|_F\le\varepsilon.
\end{gathered}
\end{equation}
The same event also gives $\|(C^{(0)})^{-1/2}\bar x\|_2\le\varepsilon$ and $\|\bar z\|_2\le\nu\varepsilon$, where $\bar z=n^{-1}\sum_i z_i$ and $\nu=\sqrt{\Tr\Omega}$. It includes
\begin{equation}\label{eq:v8-event-noise}
\|[z_i]_i/\sqrt n\|_F\le2\sqrt{\Tr\Omega},\qquad
\left\|\frac1n\sum_i z_i \tilde\xi_i^{(0)\top}\right\|_F
\le\frac{\sqrt{d\Tr\Omega}+\sqrt{2q\|\Omega\|_{\op}}}{\sqrt n}
\le\sqrt{\Tr\Omega}\,\varepsilon.
\end{equation}
For any $\|B\|_F\le1$ and $0<\gamma\le3/(16\sqrt q R)$, define
$e_i=\gamma^{-1}\nu_{H^{(0)}}(\gamma B^\top \tilde\xi_i^{(0)})-B^\top \tilde\xi_i^{(0)}$.
Then $\|[e_i]_i/\sqrt n\|_*\le\varepsilon$, for the mixed norm defined in \eqref{eq:multi-mixed-norm}.
\end{lemma}

\begin{lemma}[Pisier's Gaussian moment inequality]\label{lem:pisier-gaussian}
Let $G$ be a standard Gaussian vector, let $h$ be a smooth scalar function, and write $\|Z\|_{L^u}=(\mathbb E|Z|^u)^{1/u}$ for $u\ge2$. Whenever the displayed moments are finite,
\begin{equation}\label{eq:pisier-gaussian-moment}
\|h(G)-\mathbb Eh(G)\|_{L^u}\le\frac\pi2\sqrt u\,\big\|\|\nabla h(G)\|_2\big\|_{L^u}.
\end{equation}
Consequently,
\begin{equation}\label{eq:pisier-second-order}
\|h-\mathbb Eh\|_{L^u}\le\frac\pi2\sqrt u\left(\mathbb E\|\nabla h\|_2+\frac\pi2\sqrt u\,\big\|\|D^2h\|_{\op}\big\|_{L^u}\right).
\end{equation}
\end{lemma}

\begin{proof}
The first inequality, \eqref{eq:pisier-gaussian-moment}, is Pisier's Gaussian moment bound \citep[Theorem~2.2]{pisier1986}. To obtain \eqref{eq:pisier-second-order}, we apply \eqref{eq:pisier-gaussian-moment} to the size of the gradient. Write $v(x)=\|\nabla h(x)\|_2$. To differentiate even where $\nabla h(x)=0$, take $v_\tau(x)=(v(x)^2+\tau^2)^{1/2}$ for $\tau>0$. The chain rule gives
\[
\nabla v_\tau(x)=\frac{D^2h(x)\nabla h(x)}{v_\tau(x)},\qquad \|\nabla v_\tau(x)\|_2\le\|D^2h(x)\|_{\op}.
\]
Applying \eqref{eq:pisier-gaussian-moment} to $v_\tau$ therefore bounds $\|v_\tau(G)-\mathbb Ev_\tau(G)\|_{L^u}$ by $\frac{\pi}{2}\sqrt u\,\big\|\|D^2h(G)\|_{\op}\big\|_{L^u}$. Since $0\le v_\tau-v\le\tau$, the same bound holds for $v$ as $\tau\downarrow0$. The triangle inequality now gives
\[
\|v(G)\|_{L^u}\le\mathbb Ev(G)+\|v(G)-\mathbb Ev(G)\|_{L^u}
\le\mathbb Ev(G)+\frac{\pi}{2}\sqrt u\,\big\|\|D^2h(G)\|_{\op}\big\|_{L^u}.
\]
Substituting this bound for the gradient norm on the right of \eqref{eq:pisier-gaussian-moment} gives \eqref{eq:pisier-second-order}.
\end{proof}

\begin{proof}[Proof of Lemma~\ref{lem:v8-data-event}]
We first prove a uniform moment event on the original Gaussian sample, then transfer it through centering and whitening.
\subsubsection{Pisier concentration of the exponential-weighted moments}
For $\|b\|_2\le1/(2\sqrt q)$, define
\begin{equation}\label{eq:v8-centered-tilts}
\begin{aligned}
M_0(b)&=\frac1n\sum_i\exp(b^\top g_i),\\
Z_1(b)&=\frac1n\sum_i\exp(b^\top g_i)(g_i-b),\\
Z_2(b)&=\frac1n\sum_i\exp(b^\top g_i)\big[(g_i-b)(g_i-b)^\top-I_d\big].
\end{aligned}
\end{equation}
Completing the Gaussian square gives expectations $m_0(b)=\exp(\|b\|_2^2/2)$, zero, and zero, respectively. We apply Pisier's inequality to their scalar coordinates as functions of the full, unconditioned Gaussian sample. No keys or exponential weights are truncated.

Here are the derivative estimates used in that application. A scalar summand has the form $f_b(g)=\exp(b^\top g)P(g-b)$, where $P$ is $1$, a coordinate $v_j$, or $v_jv_k-\mathbf1_{\{j=k\}}$. For $q\ge8$ and $\|b\|_2\le1/(2\sqrt q)$,
\begin{equation}\label{eq:pisier-moment-derivatives}
\mathbb E\|\nabla_g f_b(g)\|_2^2\le5,\qquad
\big\|\|D_g^2f_b(g)\|_{\op}\big\|_{L^q}\le7.
\end{equation}
To verify these dimension-free bounds, write $\beta=\|b\|_2$ and $v=g-b$. Direct differentiation gives
\[
\nabla_gf_b=\exp(b^\top g)(bP+\nabla P),\qquad
D_g^2f_b=\exp(b^\top g)(bb^\top P+b\nabla P^\top+\nabla P b^\top+D^2P).
\]
For the first estimate, completing the square with exponent $2b^\top g$ makes $v$ Gaussian with mean $b$. The largest bound among the listed coordinates is
\[
\exp(2\beta^2)(4+14\beta^2+8\beta^4+\beta^6)<5\qquad(\beta^2\le1/32).
\]
For the second, completing the square inside the $q$th moment replaces $v$ by $g+(q-1)b$ and contributes the factor $\exp(q\beta^2/2)\le\exp(1/8)$. Each shifted coordinate obeys $\beta\|v_j\|_{L^q}\le3/4$ and $\beta\|v_j\|_{L^{2q}}\le1/\sqrt2+1/4$. The quadratic coordinates therefore give the bound
\[
\exp(1/8)\left[(1/\sqrt2+1/4)^2+1/32+4(3/4)+2\right]<7;
\]
the constant and linear coordinates give smaller bounds. This proves \eqref{eq:pisier-moment-derivatives}.

For a fixed coordinate and score vector, let $h_b(g_1,\ldots,g_n)=n^{-1}\sum_i f_b(g_i)$. Its Hessian is block diagonal, so
\[
\mathbb E\|\nabla h_b\|_2\le\sqrt{5/n},\qquad
\big\|\|D^2h_b\|_{\op}\big\|_{L^q}\le\frac{7n^{1/q}}n\le\frac{12}n.
\]
The last inequality uses $q\ge2\log n$. Equation~\eqref{eq:pisier-second-order}, followed by $n\ge100q^3$, now gives
\begin{equation}\label{eq:pisier-fixed-score}
\|h_b-\mathbb Eh_b\|_{L^q}\le\frac{\pi\sqrt5}{2}\sqrt{q/n}+3\pi^2\frac qn<4\sqrt{q/n}.
\end{equation}
Markov's inequality gives error at most $4\exp(1)\sqrt{q/n}$ with failure at most $\exp(-q)$. Applying Pisier a second time to the gradient norm is what retains the $\sqrt{q/n}$ rate; bounding each sample gradient directly in $L^q$ would lose an additional power of $q$.

\subsubsection{Uniformity over score vectors}
Take a net of radius $1/(2\sqrt q\,n^2)$ in the score ball. It has at most $(1+2n^2)^d$ points, and there are $1+d+d^2$ scalar coordinates. The definition of $q$ ensures
\[
2(1+d+d^2)(1+2n^2)^d\exp(-q)\le\delta/4.
\]
For $d=1$, the left side is at most $\delta(12+6/n^2)/64$; for $d\ge2$ the factor $(2n)^{-2d^2}$ makes it smaller.

Separately, Gaussian norm concentration gives $\max_i\|g_i\|_2\le K_g=\sqrt d+\sqrt{2q}\le2\sqrt q$ except with probability $n\exp(-q)$; see \citep{vershynin-rm}. On this event, each scalar summand's score derivative has norm at most $12\exp(1)q^{3/2}$. This follows by differentiating $\exp(b^\top g)P(g-b)$ in $b$ and using $\|g\|_2\le2\sqrt q$ and $|\exp(b^\top g)|\le\exp(1)$. The derivatives of its expectation obey the same bound. Interpolation from the net therefore costs at most $12\exp(1)q/n^2\le\tfrac1{10}\sqrt{q/n}$. Combined with \eqref{eq:pisier-fixed-score}, this proves the common uniform scalar tolerance
\begin{equation}\label{eq:v8-zeta}
\boxed{\zeta=12\sqrt{q/n},\qquad\varepsilon=48d\sqrt{q/n}}
\end{equation}
for $M_0-m_0$ and every coordinate of $Z_1,Z_2$. The bounded-key event is used only to extend the already established Gaussian moment bounds off the score net, not to condition their proof.

Let $\nu_g(b)$ and $\mathcal C_g(b)$ be the exponentially weighted mean and covariance of the $g_i$'s. Their exact cancellation identities are
\begin{equation}\label{eq:v8-tilt-cancellation}
\nu_g(b)-b=\frac{Z_1(b)}{M_0(b)},\qquad
\mathcal C_g(b)-I_d=\frac{Z_2(b)}{M_0(b)}-\frac{Z_1(b)Z_1(b)^\top}{M_0(b)^2}.
\end{equation}
Since $M_0\ge1-\zeta$ and $\zeta=\varepsilon/(4d)\le1/64$,
\begin{equation}\label{eq:v8-moment-ratios}
\sup_b\|\nu_g(b)-b\|_2\le\frac{\sqrt d\,\zeta}{1-\zeta},\qquad
\sup_b\|\mathcal C_g(b)-I_d\|_F\le\frac{d\zeta}{(1-\zeta)^2}\le16d\sqrt{q/n}=\varepsilon/3.
\end{equation}
At zero score vector, $M_0(0)=1$ exactly, so $\|\bar g\|_2\le\sqrt d\,\zeta$. Both the mean and covariance bounds are uniform over the score ball, including sample-dependent queries and self-keys.

\subsubsection{Transfer the event through centering and whitening}
Set $B^{(0)}=\bigl(C^{(0)}\bigr)^{-1/2}\Sigma^{1/2}$ only in this proof. Then
\[
C^{(0)}=\Sigma^{1/2}\mathcal C_g(0)\Sigma^{1/2},\qquad
\tilde\xi_i^{(0)}=B^{(0)}(g_i-\bar g),\qquad
B^{(0)}\mathcal C_g(0)B^{(0)\top}=I_d.
\]
It follows that $\|B^{(0)}\|_{\op}\le(1-\varepsilon/3)^{-1/2}$ and the relative covariance bounds in \eqref{eq:v8-event-cov} hold. Operator monotonicity of the positive square root gives
\[
\sqrt{1-\varepsilon/3}\,\Sigma^{1/2}
\preceq\bigl(C^{(0)}\bigr)^{1/2}
\preceq\sqrt{1+\varepsilon/3}\,\Sigma^{1/2}.
\]
Subtracting $\Sigma^{1/2}$ proves the stated operator-norm error. The mean bound gives the assertion for $\bar x$. The bounds $q\ge8d$, $K_g\le(1/\sqrt8+\sqrt2)\sqrt q$, $R\ge\sqrt{8q}$, and $\varepsilon\le1/16$ give $\max_i\|\tilde\xi_i^{(0)}\|\le3R/4$. Also
$n^{-1}\sum_i g_ig_i^\top=\mathcal C_g(0)+\bar g\bar g^\top\preceq4I_d$.

Centering changes the exponential weights only by a common normalizing factor. Consequently,
\begin{equation}\label{eq:v8-whitening-transfer}
\mathcal C_{H^{(0)}}(b)-I_d
=B^{(0)}\big[\mathcal C_g(B^{(0)\top}b)-\mathcal C_g(0)\big]B^{(0)\top}.
\end{equation}
Every $\|b\|\le3/(16\sqrt q)$ maps inside the Gaussian score-vector ball, and the two covariance errors on the right are each at most $\varepsilon/3$. Its norm is therefore at most $2\varepsilon/[3(1-\varepsilon/3)]\le\varepsilon$.

The initial normalized mean vanishes. The identity $D_b\nu_H(b)=\mathcal C_H(b)$ gives
\begin{equation}\label{eq:v8-zero-tilt-identity}
\gamma^{-1}\nu_{H^{(0)}}(\gamma B^\top \tilde\xi_i^{(0)})-B^\top \tilde\xi_i^{(0)}
=\int_0^1[\mathcal C_{H^{(0)}}(s\gamma B^\top \tilde\xi_i^{(0)})-I_d]
B^\top \tilde\xi_i^{(0)}\,ds.
\end{equation}
Exact whitening controls its RMS norm by $\varepsilon\|B\|_F$, and the token-radius bound controls its scaled maximum norm by $\varepsilon\|B\|_{\op}$. This proves the mean assertion without a $1/\gamma$ sample-mean term.

Conditional on the inputs, the vectorization of $n^{-1}\sum_i z_i \tilde\xi_i^{(0)\top}$ is Gaussian with covariance $I_d\otimes\Omega/n$. Gaussian norm concentration proves the cross-moment bound in \eqref{eq:v8-event-noise} with failure at most $\exp(-q)$. Applying the same inequality to $[z_i]_i/\sqrt n$ and using $n\ge2q$ gives its first bound with failure at most $\exp(-q)$. The final comparison with $\sqrt{\Tr\Omega}\,\varepsilon$ follows from $q\ge8d$ and $\|\Omega\|_{\op}\le\Tr\Omega$. The whitened sample-mean bound follows from $\|\bar g\|\le\sqrt d\zeta$ and $\|(C^{(0)})^{-1/2}\Sigma^{1/2}\|_{\op}\le(1-\varepsilon/3)^{-1/2}$. Applying Gaussian norm concentration to $\bar z\sim N(0,\Omega/n)$ gives $\|\bar z\|\le(\sqrt{\Tr\Omega}+\sqrt{2q\|\Omega\|_{\op}})/\sqrt n\le\nu\varepsilon$, with failure at most $\exp(-q)$. These arguments hold for $d$-dimensional noise; the cross-moment covariance remains $I_d\otimes\Omega/n$. Altogether the failure probability is at most
\[
\delta/4+(n+3)\exp(-q)<\delta/2.
\]
This proves Lemma~\ref{lem:v8-data-event}.
\end{proof}

\subsection{Uniform concentration of the softmax mean and covariance}\label{sec:softmax-statistics}
The next two lemmas isolate the attention statistics used in the layer induction. Assume $x_1,\ldots,x_n\stackrel{\mathrm{iid}}{\sim}\mathcal N(0,\Sigma)$ with $\Sigma\succ0$, and use $q,R,\varepsilon$ from \eqref{eq:choices}, with $n\ge128$, $n\ge100q^3$, and $\varepsilon\le1/16$. Write $C=C^{(0)}$, $\tilde\xi_i=C^{-1/2}(x_i-\bar x)$, and $\zeta=12\sqrt{q/n}$. Let $\mathcal E_x$ be the input-sample event constructed in the proof of Lemma~\ref{lem:v8-data-event}, before the label-noise bounds are imposed. It has probability at least $1-\delta/2$ and supplies \eqref{eq:v8-event-cov}, \eqref{eq:v8-moment-ratios}, and $\|\Sigma^{-1/2}\bar x\|_2\le\sqrt d\,\zeta$.

For every matrix $B$ and scale $\gamma$ satisfying
\[
\|B\|_F\le1,\qquad 0<\gamma\le\frac{3}{16\sqrt q R},\qquad V=\gamma C^{-1/2}BC^{-1/2},
\]
let $\pi_{ij},\mu_i,\Sigma_i$ be the centered empirical statistics in \eqref{eq:attention}, including the self-key $j=i$. For a fixed query $x$ and an independent population key $X\sim\mathcal N(0,\Sigma)$, completing the square gives the Gaussian mean and covariance under exponential weights
\[
\mu_{\mathrm{pop}}(x;V)=\frac{\mathbb E[X\exp(x^\top VX)]}{\mathbb E[\exp(x^\top VX)]}=\Sigma V^\top x,
\qquad \Sigma_{\mathrm{pop}}(x;V)=\Sigma.
\]
Both lemmas below hold on the same event $\mathcal E_x$, simultaneously for all $i\in[n]$ and all admissible $(B,\gamma)$; these parameters may therefore depend on the sample. The probability is over the original sample, rather than a separate conditional guarantee for every fixed value of $x_i$.

\begin{lemma}[Softmax covariance concentration]\label{lem:softmax-covariance-concentration}
Under the preceding assumptions, on $\mathcal E_x$,
\begin{equation}\label{eq:softmax-covariance-concentration}
\max_{i\in[n]}\|\Sigma^{-1/2}\Sigma_i\Sigma^{-1/2}-I_d\|_F\le\frac{\varepsilon}{3},
\qquad
\max_{i\in[n]}\|C^{-1/2}\Sigma_iC^{-1/2}-I_d\|_F\le\varepsilon.
\end{equation}
In particular, $\max_i\|\Sigma_i-\Sigma\|_F\le(\varepsilon/3)\|\Sigma\|_{\op}$.
\end{lemma}
\begin{proof}
Put $g_j=\Sigma^{-1/2}x_j$ and define the two score vectors
\[
\beta_i=\gamma B^\top\tilde\xi_i=C^{1/2}V^\top(x_i-\bar x),
\qquad b_i=\Sigma^{1/2}V^\top(x_i-\bar x)=\Sigma^{1/2}C^{-1/2}\beta_i.
\]
By \eqref{eq:v8-event-cov}, $\|\beta_i\|_2\le3/(16\sqrt q)$ and
$\|b_i\|_2\le(1-\varepsilon/3)^{-1/2}3/(16\sqrt q)<1/(2\sqrt q)$.
Centering the keys subtracts the same scalar from every logit, so
\[
\Sigma_i=\Sigma^{1/2}\mathcal C_g(b_i)\Sigma^{1/2},\qquad
C^{-1/2}\Sigma_iC^{-1/2}=\mathcal C_{H^{(0)}}(\beta_i).
\]
The uniform bounds \eqref{eq:v8-moment-ratios} and \eqref{eq:v8-event-cov} prove \eqref{eq:softmax-covariance-concentration}. Multiplying the first covariance error on both sides by $\Sigma^{1/2}$ gives the final assertion.
\end{proof}

\begin{lemma}[Softmax mean concentration]\label{lem:softmax-mean-concentration}
Under the same assumptions, on $\mathcal E_x$,
\begin{equation}\label{eq:softmax-mean-population}
\max_{i\in[n]}\|\Sigma^{-1/2}(\mu_i-\mu_{\mathrm{pop}}(x_i;V))\|_2
\le\left(\frac1{1-\zeta}+1+\frac{\gamma}{1-\varepsilon/3}\right)\sqrt d\,\zeta.
\end{equation}
The centered, whitened mean also satisfies a bound proportional to the score scale. Specifically, set
\[
e_i=\frac1\gamma C^{-1/2}\big[\mu_i-CV^\top(x_i-\bar x)\big]
=\frac1\gamma C^{-1/2}\mu_i-B^\top\tilde\xi_i.
\]
Then, simultaneously for every $i\in[n]$,
\begin{equation}\label{eq:softmax-mean-whitened}
\|e_i\|_2\le\varepsilon\|B^\top\tilde\xi_i\|_2,
\qquad \left\|[e_i]_i/\sqrt n\right\|_F\le\varepsilon\|B\|_F\le\varepsilon.
\end{equation}
\end{lemma}
\begin{proof}
Use $b_i$ from the covariance proof and $\bar g=\Sigma^{-1/2}\bar x$. The exact identity $\mu_i=\Sigma^{1/2}\nu_g(b_i)-\bar x$ gives
\[
\Sigma^{-1/2}(\mu_i-\Sigma V^\top x_i)
=\nu_g(b_i)-b_i-\big(I_d+\Sigma^{1/2}V^\top\Sigma^{1/2}\big)\bar g.
\]
Here $\|\Sigma^{1/2}V^\top\Sigma^{1/2}\|_{\op}\le\gamma/(1-\varepsilon/3)$. Applying \eqref{eq:v8-moment-ratios} and $\|\bar g\|_2\le\sqrt d\,\zeta$ proves \eqref{eq:softmax-mean-population}, including the empirical-centering correction.

For the second assertion, apply the integral identity \eqref{eq:v8-zero-tilt-identity} and the uniform whitened covariance bound. This gives the pointwise estimate in \eqref{eq:softmax-mean-whitened}; exact whitening gives its RMS version:
\[
\frac1n\sum_i\|e_i\|_2^2
\le\frac{\varepsilon^2}{n}\sum_i\|B^\top\tilde\xi_i\|_2^2
=\varepsilon^2\|B\|_F^2,
\]
which proves its Frobenius bound. This argument retains the score-scale factor $\gamma$ without introducing a $\bar x/\gamma$ term.
\end{proof}

\subsection{Propagation through empirical layers}\label{sec:v8-stability}
For $D=[d_i]_i/\sqrt n$ and radius $R=R_n$ from \eqref{eq:choices}, set
\begin{equation}\label{eq:multi-mixed-norm}
\|D\|_*=\max\{\|D\|_F,\max_i\|d_i\|_2/R\}.
\end{equation}
\subsubsection{The normalized attention operator}
Use the token-array and mixed-norm conventions in \eqref{eq:multi-mixed-norm}. A centered-white array satisfies $H\mathbf1=0$ and $HH^\top=I_d$; a subwhite array has $HH^\top\preceq I_d$. All arrays in the following stability calculations have $\max_i\|\tilde\xi_i\|\le R$. For fixed $\|{\mathsf B}\|_F\le1$, let the weights be proportional to $\exp(\gamma \tilde\xi_i^\top {\mathsf B}\tilde\xi_j)$, let $\mathcal A_i(H,{\mathsf B})=\sum_j\pi_{ij}\tilde\xi_j$, and let $\mathcal C_i(H,{\mathsf B})$ be their weighted covariance. Write
\begin{equation}\label{eq:multi-normalized-error}
e(H,{\mathsf B})=\frac1\gamma[\mathcal A_i(H,{\mathsf B})]_i/\sqrt n-{\mathsf B}^\top H.
\end{equation}
The normalized statistics are related to the network's unwhitened statistics by
\begin{equation}\label{eq:multi-unwhitened-convention}
\begin{aligned}
\mathcal A_i(H^{(\ell-1)},{\mathsf B}^{(\ell)})
&=P^{(\ell-1)\top}(C^{(\ell-1)})^{-1/2}\mu_i^{(\ell)},\\
\mathcal C_i(H^{(\ell-1)},{\mathsf B}^{(\ell)})
&=P^{(\ell-1)\top}(C^{(\ell-1)})^{-1/2}\Sigma_i^{(\ell)}
(C^{(\ell-1)})^{-1/2}P^{(\ell-1)},\\
{\mathsf B}^{(\ell)}&=\tilde V^{(\ell)}/r^{(\ell)}.
\end{aligned}
\end{equation}

\begin{lemma}[Mean-error and covariance stability]\label{lem:multi-stability}
Suppose $H,H'$ are centered-white, $\|{\mathsf B}\|_F\le1$, and $\gamma R^2\|{\mathsf B}\|_{\op}\le1/4$. Then
\begin{equation}\label{eq:multi-stability}
\begin{aligned}
\|e(H,{\mathsf B})-e(H',{\mathsf B})\|_*&\le7\gamma R^2\|{\mathsf B}\|_{\op}^2\|H-H'\|_*,\\
\max_i\|\mathcal C_i(H,{\mathsf B})-\mathcal C_i(H',{\mathsf B})\|_{\op}&\le4\|H-H'\|_*.
\end{aligned}
\end{equation}
The first bound also holds with Frobenius norms throughout. Its derivative form holds at a white array for every centered tangent direction $D$ satisfying $DH^\top+HD^\top=0$. At fixed $H$,
\begin{equation}\label{eq:multi-fixed-token}
D_{\mathsf B}e_i[\dot {\mathsf B}]=(\mathcal C_i-I_d)\dot {\mathsf B}^\top \tilde\xi_i,\qquad
\|D_{\mathsf B}e[\dot {\mathsf B}]\|_F\le\max_i\|\mathcal C_i-I_d\|_{\op}\|\dot {\mathsf B}\|_F.
\end{equation}
\end{lemma}
\begin{proof}
For a centered subwhite array, Jensen gives $\pi_{ij}\le \exp(1/4)/n$ and $\mathcal C_i\preceq \exp(1/4)I_d$. Subtract the correct nonwhite response and differentiate
$\gamma^{-1}\mathcal A(H,{\mathsf B})-(HH^\top){\mathsf B}^\top H$.
The zero-order terms cancel, leaving four token arrays:
\begin{equation}\label{eq:multi-four-terms}
\begin{aligned}
&\sum_j\left[\frac{\pi_{ij}-n^{-1}}\gamma-\frac{\tilde\xi_i^\top {\mathsf B}\tilde\xi_j}{n}\right]d_j,\qquad
(\mathcal C_i-HH^\top){\mathsf B}^\top d_i,\\
&\sum_j(\pi_{ij}-n^{-1})\tilde\xi_j(\tilde\xi_i^\top {\mathsf B}d_j),\qquad
-\mathcal A_i\sum_j\pi_{ij}\tilde\xi_i^\top {\mathsf B}d_j.
\end{aligned}
\end{equation}
Their norms, divided by $\gamma R^2\|{\mathsf B}\|_{\op}^2\|D\|$, are bounded by $\exp(1/4),2,2,\exp(1/2)$, respectively, in both the mixed and Frobenius norms. Here are the estimates needed for these numerical constants. The row mean of the scores $\tilde\xi_i^\top {\mathsf B}\tilde\xi_j$ is zero; their squared row and column sums are at most $n\|{\mathsf B}\|_{\op}^2R^2$. Taylor's remainder $|\exp(u)-1-u|\le \exp(1/4)u^2/2$ controls the first array, including the denominator correction; the constant-in-$j$ term vanishes because $D$ is centered. The density ratios obey $|n\pi_{ij}-1|\le2\gamma R^2\|{\mathsf B}\|_{\op}$. Weighted-variance comparison gives
$\exp(-2\gamma R^2\|{\mathsf B}\|_{\op})HH^\top\preceq\mathcal C_i\preceq \exp(\gamma R^2\|{\mathsf B}\|_{\op})HH^\top$, yielding the second bound. For the third array, use
$\|\sum_j(\pi_{ij}-n^{-1})\tilde\xi_j(\tilde\xi_i^\top {\mathsf B}d_j)\|
\le2\gamma R^2\|{\mathsf B}\|_{\op}\|D^\top {\mathsf B}^\top \tilde\xi_i\|$.
For the fourth, integrate the exponentially weighted covariance to obtain
$\|\mathcal A_i\|\le \exp(1/4)\gamma\|{\mathsf B}\|_{\op}\|\tilde\xi_i\|$, and use
$|\sum_j\pi_{ij}\tilde\xi_i^\top {\mathsf B}d_j|\le \exp(1/4)\|D^\top {\mathsf B}^\top \tilde\xi_i\|$.
Averaging squared norms uses $HH^\top\preceq I_d$; dividing pointwise bounds by $R$ gives the other component. Since $\exp(1/4)+4+\exp(1/2)<7$, integration along the straight subwhite segment proves the mean-error difference bound. At a white point the tangent condition makes the derivative of $HH^\top$ zero, giving the derivative version.

For the covariance, put $v_{ij}=\tilde\xi_j-\mathcal A_i$ only within this calculation. Direct differentiation gives
\[
\dot{\mathcal C}_i=\sum_j\pi_{ij}(d_jv_{ij}^\top+v_{ij}d_j^\top)
+\sum_j\pi_{ij}(v_{ij}v_{ij}^\top-\mathcal C_i)
\gamma(d_i^\top {\mathsf B}\tilde\xi_j+\tilde\xi_i^\top {\mathsf B}d_j).
\]
The first term has operator norm at most $2\exp(1/4)\|D\|_F$, by multiplying a weighted token matrix in operator norm by a weighted variation matrix in Frobenius norm. The scalar score variation is at most $2\gamma R^2\|{\mathsf B}\|_{\op}\|D\|_*$. Testing the second matrix on a unit vector bounds it by $4\exp(1/4)\gamma R^2\|{\mathsf B}\|_{\op}\|D\|_*$. Their sum is at most $3\exp(1/4)\|D\|_*<4\|D\|_*$. Integrate along the subwhite segment. Finally differentiate the weights at fixed tokens to obtain \eqref{eq:multi-fixed-token}; exact whitening converts its averaged squared norm to $\|\dot {\mathsf B}\|_F^2$.
\end{proof}

\subsubsection{An exact normalization identity and its differential}
\begin{lemma}[Whitening and alignment of a perturbed linear map]\label{lem:multi-normalization}
Let a centered output matrix be $\mathsf X=(T+A)H+E$, where $H=[\tilde\xi_i]_i/\sqrt n$ is centered, $HH^\top=I_d$, and $\max_i\|\tilde\xi_i\|_2\le R$. All token-matrix columns are divided by $\sqrt n$. Suppose $A_{\rm fit}=\mathsf XH^\top$ is invertible and define
\[
A_{\rm fit}=\mathsf XH^\top=T+A+EH^\top,\qquad
N=E(I_n-H^\top H),\qquad \mathsf{Z}=A_{\rm fit}^{-1}N.
\]
Here $E$ is the centered nonlinear remainder. The actual whitening-and-alignment operation is exactly
\begin{equation}\label{eq:multi-normalization-identity}
H_+=(I_d+\mathsf{Z}\mathsf{Z}^\top)^{-1/2}(H+\mathsf{Z}),\qquad
T_+=A_{\rm fit}(I_d+\mathsf{Z}\mathsf{Z}^\top)^{1/2}.
\end{equation}
Suppose
\[
\sigma_{\min}(A_{\rm fit})\ge m/2,\qquad \|A_{\rm fit}\|_{\op}\le2(M+m),\qquad
\frac{2\|E\|_*}{m}\le\frac1{64},\qquad
\frac{8(M+m)\|E\|_*}{m^2}\le\frac1{16}.
\]
Then
\begin{equation}\label{eq:multi-normalization-bounds}
\begin{aligned}
\|\mathsf{Z}\|_F&\le2\|E\|_*/m,&\|\mathsf{Z}\|_*&\le4\|E\|_*/m,\\
\|H_+-H\|_*&\le4\|E\|_*/m+2\|E\|_*^2/m^2,&
\|T_+-A_{\rm fit}\|_F&\le4(M+m)\|E\|_*^2/m^2.
\end{aligned}
\end{equation}
If $S_+=S+A$ is the linear reference, differentiation along a family of centered-white arrays $H$ also gives
\begin{equation}\label{eq:multi-normalization-derivative}
\begin{aligned}
\|\dot H_+\|_F+\frac2m\|\dot T_+-\dot S_+\|_F
\le{}&\left(1+\frac{8\|E\|_*}{m}\right)
\left(\|\dot H\|_F+\frac2m\|\dot T-\dot S\|_F\right)\\
&+\frac6m\|\dot E\|_F+\frac{8\|E\|_*}{m^2}\|\dot S_+\|_F.
\end{aligned}
\end{equation}
\end{lemma}
\begin{proof}
The decomposition $\mathsf X=A_{\rm fit}H+N$ is the least-squares fit of the new tokens to the preceding white tokens, with $NH^\top=0$. Thus $\mathsf X=A_{\rm fit}(H+\mathsf Z)$ and $\mathsf ZH^\top=0$. Put $D=(I_d+\mathsf Z\mathsf Z^\top)^{1/2}$. Then
\[
\mathsf X\mathsf X^\top=A_{\rm fit}D^2A_{\rm fit}^\top,\qquad
\polar\big((\mathsf X\mathsf X^\top)^{-1/2}\mathsf XH^\top\big)=(\mathsf X\mathsf X^\top)^{-1/2}A_{\rm fit}D.
\]
Substitution into the whitening-and-alignment rule gives \eqref{eq:multi-normalization-identity}. In particular, if $E=0$, normalization returns $H_+=H$ and transfers the entire linear update into $T_+=T+A$. Only the orthogonal remainder $N$ changes the normalized tokens. Invertibility of $A_{\rm fit}$ also makes the covariance positive definite and the polar factor unique.

Centering has mixed-norm operator norm at most two, and right multiplication by $H^\top H$ has mixed-norm operator norm at most one. Together with $\|A_{\rm fit}^{-1}\|_{\op}\le2/m$ and $\|(I_d+\mathsf{Z}\mathsf{Z}^\top)^{1/2}-I_d\|_F\le\|\mathsf{Z}\|_F^2/2$, these estimates give \eqref{eq:multi-normalization-bounds}.

For completeness, these derivative constants follow by differentiating the inverse and the square root in \eqref{eq:multi-normalization-identity}:
\[
\dot{\mathsf{Z}}=A_{\rm fit}^{-1}\left[\dot E(I_n-H^\top H)-E(\dot H^\top H+H^\top\dot H)-\dot A_{\rm fit}\mathsf{Z}\right].
\]
The square root $D$ satisfies
$D\dot D+\dot DD=\dot{\mathsf{Z}}\mathsf{Z}^\top+\mathsf{Z}\dot{\mathsf{Z}}^\top$ and
$\|\dot D\|_F\le\|\mathsf{Z}\|_F\|\dot{\mathsf{Z}}\|_F$ because $D\succeq I_d$.
Put $u=2\|E\|_*/m$ for this numerical check. In the sum of the derivatives of $H_+$ and $T_+$, the multiplier of $\|\dot{\mathsf{Z}}\|_F$ is at most
$1+u+8(M+m)\|E\|_*/m^2\le69/64$.
Using $\sqrt{1+u^2}\le1+u^2/2$ and $u\le1/64$, the four coefficients of
$\|\dot H\|_F$, $2\|\dot T-\dot S\|_F/m$, $2\|\dot E\|_F/m$, and $2\|\dot S_+\|_F/m$
are respectively bounded by $1+4u,1+4u,3,2u$. This proves \eqref{eq:multi-normalization-derivative}, including the derivative of the raw coordinate map $T_+$.

\end{proof}

\input{body/reference_aliases}

\clearpage
\section{Experiments}\label{sec:experiments-appendix}

Figure~\ref{fig:lorenz-training-test-stacked} presents the training and test curves from Figure~\ref{fig:lorenz-training-test} in a stacked layout.

\begin{figure}[htbp]
\centering
\includegraphics[width=\textwidth]{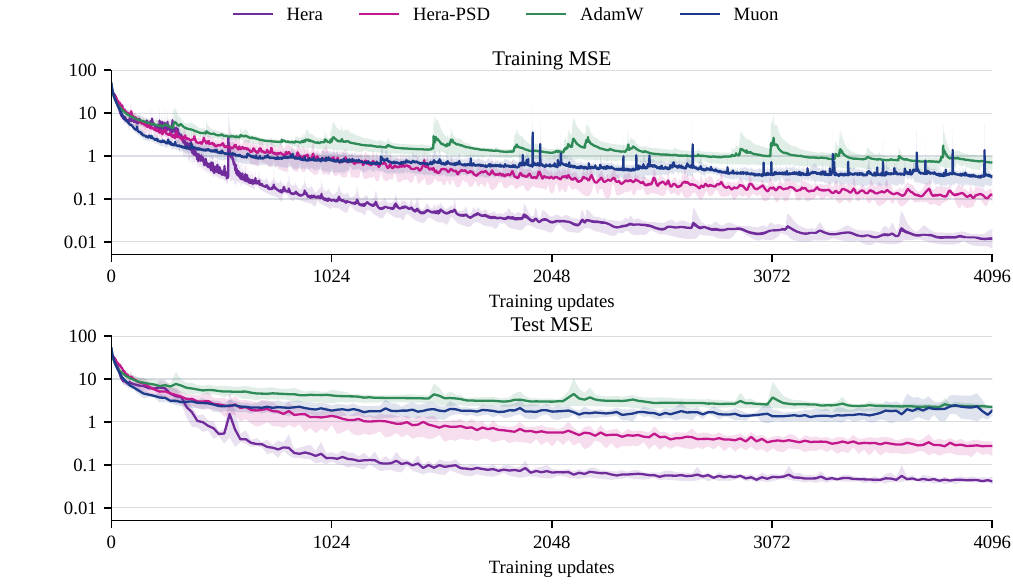}
\caption{Training MSE (top) and test MSE (bottom) for the validation-selected configurations. Lines give means over five seeds and shading gives their range. Training loss is recorded at every update. We measure test MSE at initialization, every 5 updates through update 100, every 25 updates thereafter, and at the final update (4,096).}
\label{fig:lorenz-training-test-stacked}
\end{figure}



\paragraph{Forecasting task.}
We simulate Lorenz--63 with $(\sigma,\rho,\beta)=(10,28,8/3)$ from $(1,1,1)$ using fourth-order Runge--Kutta with step $0.001$. We retain one state every $0.1$ time units, with no burn-in or added noise. Each $15$-dimensional token stacks five observed states from oldest to newest. Its target is the next state. The training, validation, and test splits each contain 500 targets, with observation indices $5$--$504$, $510$--$1009$, and $1015$--$1514$, respectively, counting from zero. Histories overlap within a split, and the underlying observation blocks are disjoint across splits. Each held-out query attends to the 500 training inputs and itself. Every forecast uses its observed history. We report MSE in the original coordinates, averaging squared errors over all targets and all three coordinates.

\paragraph{Models and training.}
All four models have three residual attention layers and a fixed readout $W=[0_{3\times12}\; I_3]$ that selects the last three coordinates of the final-layer query token. Thus only the attention layers can improve on the unmodified input readout. Hera and Hera-PSD differentiate through empirical centering, covariance whitening, polar alignment, and score normalization. Hera takes gradient steps directly in independent lifted matrices $(\tilde U,\tilde V)$. Hera-PSD instead updates $(B,K)$ through $\tilde V=B$ and $\tilde U=B S(K)$, where $S(K)=(I+K^\top K)^{-1}$. Its lifted increment $B S(K)B^\top$ is PSD in every layer. The experiment uses a unit residual multiplier and omits any regularization.
AdamW and Muon update independent raw attention matrices using centered tokens. AdamW uses $(\beta_1,\beta_2)=(0.9,0.999)$, $\epsilon=10^{-8}$, and decoupled weight decay. Muon uses momentum and five Newton--Schulz iterations for each update direction.

All methods train in single precision using full-batch gradients and a single learning rate $\eta_{\mathrm{hid}}$ shared across layers. Hera initializes its lifted matrices, and AdamW and Muon their raw matrices, with independent centered Gaussian entries of standard deviation $1/15$. Hera-PSD initializes $B$ at the same Gaussian scale and sets $K=I$, giving $\tilde U=B/2$ and $\tilde V=B$ at initialization.

\paragraph{Validation and testing.}
Table~\ref{tab:lorenz-hyperparameter-grid} gives 44 configurations per method. In stage~1, we train every configuration for up to 500 updates from the same five seeds (1729--1733). A run stops if its training MSE exceeds twice its own initial value. A configuration is eligible only if all five runs succeed. We rank eligible configurations by their mean validation MSE at the final update. In stage~2, we restart the ten highest-ranked configurations per method from five fresh seeds (2730--2734) and apply the same eligibility and ranking rules at 4,096 updates. We evaluate the winning configuration for each method on the test split using those same five stage~2 checkpoints. To measure test MSE between saved checkpoints, we replay the selected runs from the preceding checkpoint using the original training data. We verify that the parameters and optimizer state match exactly at the next saved checkpoint.

In stage~1, five of 880 runs hit the loss threshold (one Hera, one AdamW, three Muon). In stage~2, 20 of 200 runs hit it (eight Hera-PSD, six AdamW, six Muon). This left ten eligible Hera configurations, six Hera-PSD, eight AdamW, and eight Muon in stage~2. There were no numerical failures or replacement seeds.

\begin{table}[htbp]
\centering
\small
\begin{adjustbox}{max width=\linewidth}
\begin{tabular}{ll}
\toprule
Hyperparameter & Grid values \\
\midrule
Shared rate $\eta_{\mathrm{hid}}$ & $\{10^{-7},3\!\times\!10^{-7},10^{-6},3\!\times\!10^{-6},10^{-5},3\!\times\!10^{-5},10^{-4},3\!\times\!10^{-4},10^{-3},3\!\times\!10^{-3},10^{-2}\}$ \\
Hera and Hera-PSD score scale $\gamma$ & $\{0.1,0.3,1,3\}$ \\
AdamW weight decay & $\{0.0001,0.001,0.01,0.1\}$ \\
Muon momentum & $\{0.8,0.9,0.95,0.99\}$ \\
\bottomrule
\end{tabular}
\end{adjustbox}
\caption{Each method uses the same eleven learning rates crossed with its four method-specific values. The readout is fixed for all methods.}
\label{tab:lorenz-hyperparameter-grid}
\end{table}

\begin{table}[htbp]
\centering
\small
\setlength{\tabcolsep}{5pt}
\begin{tabular}{lclrrr}
\toprule
Method & $\eta_{\mathrm{hid}}$ & Selected hyperparameter & Train MSE & Valid. MSE & Test MSE \\
\midrule
Hera & $10^{-2}$ & $\gamma=0.1$ & $0.01178$ & $0.07966$ & $0.04135\pm0.00721$ \\
Hera-PSD & $10^{-2}$ & $\gamma=1$ & $0.12392$ & $0.63480$ & $0.27297\pm0.07067$ \\
AdamW & $10^{-3}$ & Weight decay $0.01$ & $0.71121$ & $2.80068$ & $2.22217\pm0.11721$ \\
Muon & $3\times10^{-3}$ & Momentum $0.9$ & $0.33622$ & $2.38983$ & $1.83235\pm0.46155$ \\
\midrule
Ridge & --- & --- & $6.95038$ & $9.71008$ & $7.98942$ \\
Persistence & --- & --- & $30.80135$ & $34.16136$ & $34.73314$ \\
\bottomrule
\end{tabular}
\caption{Validation-selected configurations after 4,096 updates. Each MSE is the mean of five individual seed MSEs. The test column also gives one sample standard deviation. The two deterministic baselines use the same observation splits.}
\label{tab:lorenz-selected}\label{tab:lorenz-results}
\end{table}

\paragraph{Results.}
The ordering on validation carries over to test (Table~\ref{tab:lorenz-results}). Relative to AdamW and Muon, Hera reduces mean test MSE by $98.1\%$ and $97.7\%$, respectively, and Hera-PSD by $87.7\%$ and $85.1\%$. All four improve on persistence, which copies the last observation, and ridge regression, whose penalty was selected by five chronological folds within training. The variation across seeds reflects initialization on one previously studied trajectory.

Figures~\ref{fig:lorenz-state-predictions} and~\ref{fig:lorenz-state-predictions-magnified} show the full test trajectory and a magnified window. Averaging the $x$-coordinate MSE of each individual run over seeds gives $0.01420$ for Hera, $0.19114$ for Hera-PSD, $1.10862$ for AdamW, and $1.08651$ for Muon.

\begin{figure}[p]
\centering
\includegraphics[width=\textwidth]{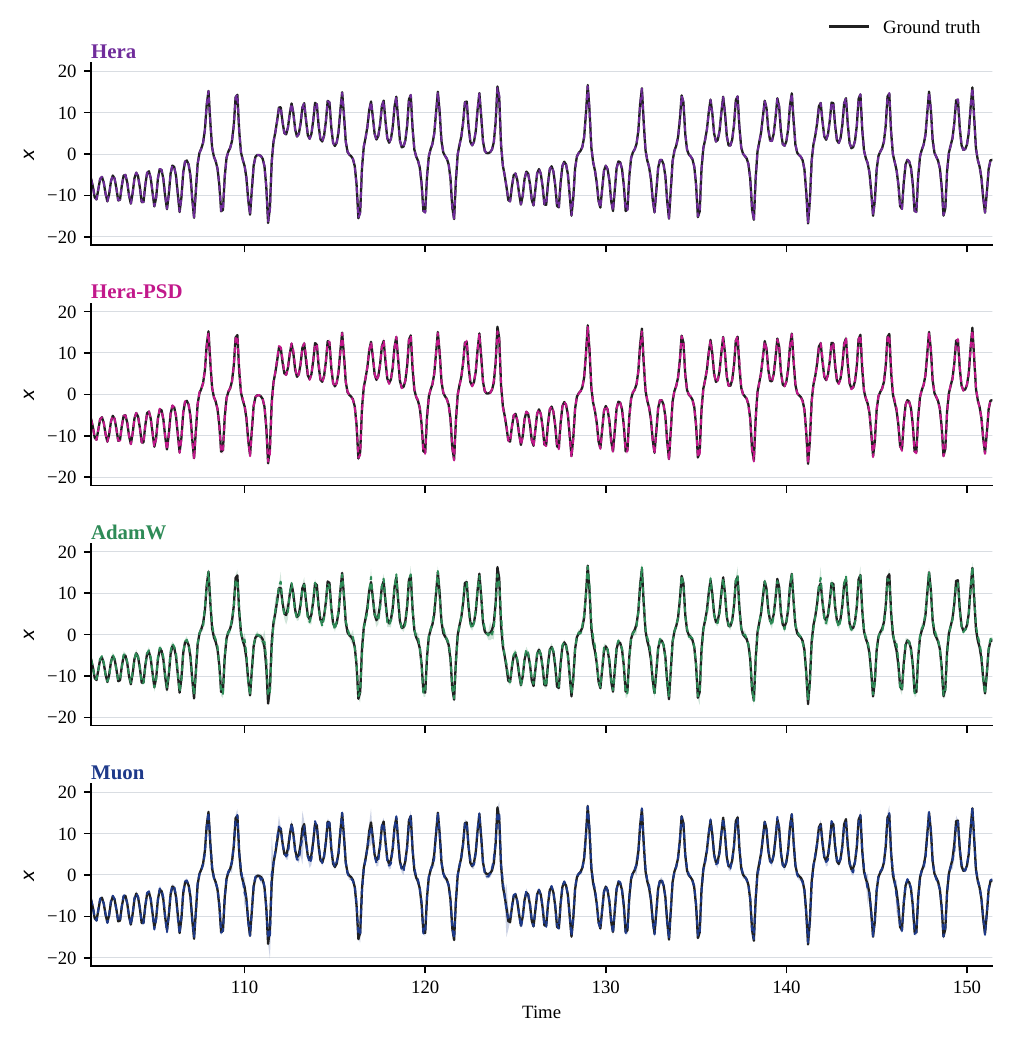}
\caption{One-step forecasts of the Lorenz--63 $x$-coordinate for all 500 test targets. Black lines show the observed trajectory. Each colored dashed line is the mean prediction across five seeds after 4,096 updates, with shading for one sample standard deviation.}
\label{fig:lorenz-state-predictions}
\end{figure}

\begin{figure}[p]
\centering
\includegraphics[width=\textwidth]{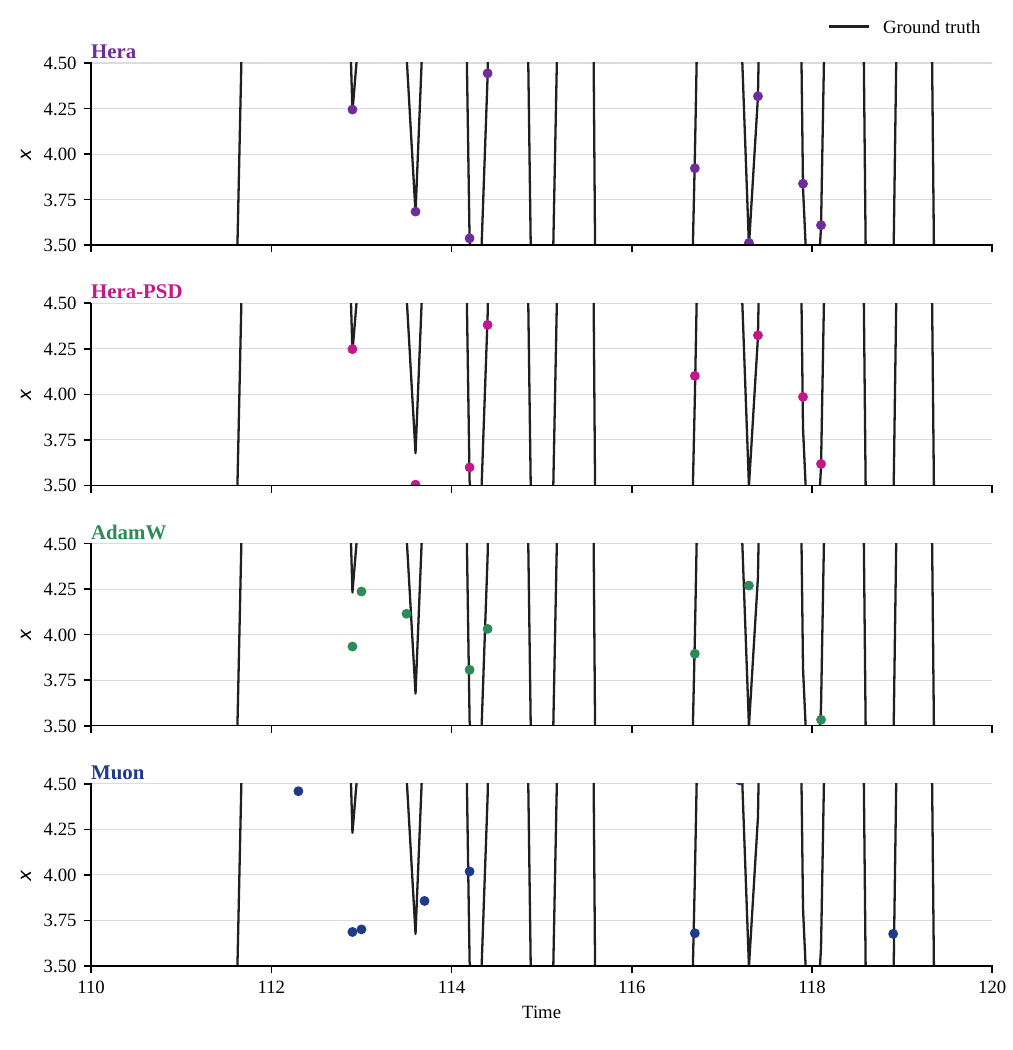}
\caption{Detail of Figure~\ref{fig:lorenz-state-predictions} for times $110$--$120$ and $x\in[3.5,4.5]$. Black lines show the observed trajectory and colored points show the five-seed mean forecast at each observation within view. All panels use the same axes.}
\label{fig:lorenz-state-predictions-magnified}
\end{figure}

\end{document}